\documentclass{article}
\PassOptionsToPackage{numbers, sort&compress}{natbib}
\usepackage[final]{neurips_2021} 

\usepackage{algorithm}
\usepackage{algorithmic}
\usepackage{amsmath}
\usepackage{amssymb}
\usepackage{bbm}
\usepackage{multirow}
\usepackage{scalerel}
\usepackage[font=small]{caption}
\usepackage{subcaption}
\usepackage{wrapfig}

\usepackage[utf8]{inputenc} 
\usepackage[T1]{fontenc}    
\usepackage[hidelinks]{hyperref}

\usepackage{url}            
\usepackage{booktabs}       
\usepackage{amsfonts}       
\usepackage{nicefrac}       
\usepackage{microtype}      
\usepackage{xcolor}         
\usepackage{ourstyle}
\usepackage{multibib} 
\newcites{si}{Additional References for the Appendix}

\newcommand{\expectation}{\mathbb{E}}

\newcommand{\alglinelabel}{%
  \addtocounter{ALC@line}{-1}
  \refstepcounter{ALC@line}
  \label
}

\newcommand{\algfullname}{Continuous Doubly Constrained Batch \\ Reinforcement Learning}
\newcommand{\algpartname}{Continuous Doubly Constrained Batch RL}

\newcommand{\algname}{CDC}

\newcommand{\beginsupplement}{ 
\setcounter{section}{0}
\renewcommand{\thesection}{S\arabic{section}} %
\renewcommand{\thesubsection}{\thesection.\arabic{subsection}}
\setcounter{table}{0}
\renewcommand{\thetable}{S\arabic{table}} %
\setcounter{figure}{0}
\renewcommand{\thefigure}{S\arabic{figure}} %
}

\newcommand{\EE}{\mathbb{E}}
\newcommand{\argmax}{\mathop{\mathrm{argmax}}}
\newcommand{\argmin}{\mathop{\mathrm{argmin}}}
\newcommand{\regular}{baseline} 
\title{\algfullname}

\author{%
  Rasool Fakoor$^1$, Jonas Mueller$^1$, Kavosh Asadi$^1$, Pratik Chaudhari$^{1,2}$, Alexander J. Smola$^1$ \\$^1$Amazon Web Services, $^2$University of Pennsylvania\\
  \texttt{fakoor@amazon.com}
}

\graphicspath{{./fig/}}

\begin{document}
\maketitle

\begin{abstract}
Reliant on too many experiments to learn good actions, current Reinforcement Learning (RL) algorithms have limited applicability in real-world settings, which can be too expensive to allow exploration. We propose an algorithm for batch RL, where effective policies are learned using only a fixed offline dataset instead of online interactions with the environment. The limited data in batch RL produces inherent uncertainty in value estimates of states/actions that were insufficiently represented in the training data. This leads to particularly severe extrapolation when our candidate policies diverge from one that generated the data. We propose to mitigate this issue via two straightforward penalties: a policy-constraint to reduce this divergence and a value-constraint that discourages overly optimistic estimates. Over a comprehensive set of $32$ continuous-action batch RL benchmarks, our approach compares favorably to state-of-the-art methods, regardless of how the offline data were collected.
\end{abstract}

\section{Introduction}
\label{sec:intro}

\begin{wrapfigure}{r}{0.50\textwidth}
\vspace*{-1.8em}
\centering
\includegraphics[width=0.25\textwidth]{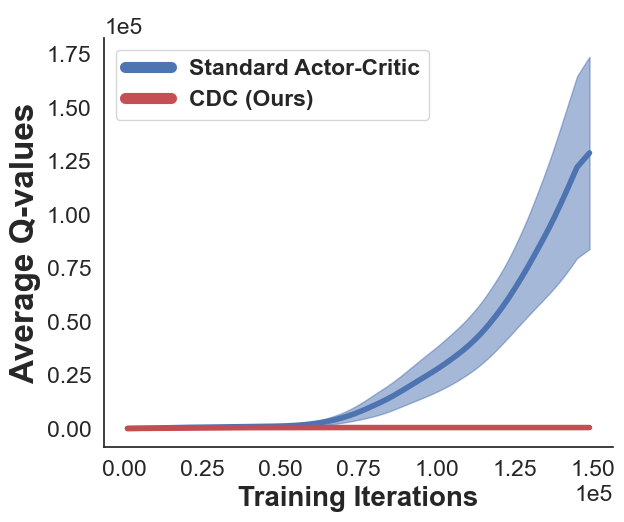}
\includegraphics[width=0.23\textwidth]{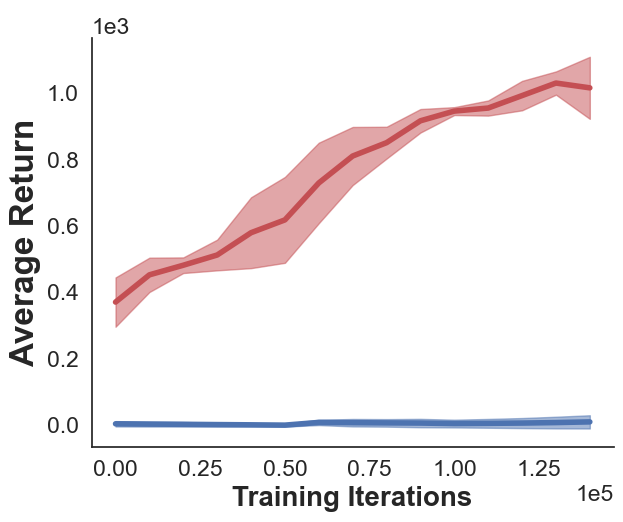}
\caption{\small{\textbf{Batch RL with CDC vs.\ No CDC}. Left: Standard actor-critic  overestimates $Q$-values whereas CDC estimates are well controlled. Right: Wild overestimation leads to worse-performing policies whereas CDC performs well. 
}}
\label{fig:extrapolate_main}
\vskip -1.28em
\end{wrapfigure} 

Deep RL algorithms have demonstrated impressive performance in simulable digital environments like video games~\citep{mnih2015human, Silver1140, SilveralphaGo}. In these settings, the agent can execute different policies and observe their performance.
Barring a few examples~\citep{li2017deep}, advancements have not translated quite as well to real-world environments, where it is typically infeasible to experience millions of environmental interactions~\citep{dulac2019challenges}. Moreover, in presence of an acceptable heuristic, it is inappropriate to deploy an agent that learns from scratch hoping that it may eventually outperform the heuristic after sufficient experimentation.

The setting of \emph{batch} or \emph{offline} RL instead offers a more pertinent framework to learn performant policies for real-world applications \citep{Lange2012, thomas2015high}. Batch RL is widely applicable because this setting does not require that: a proposed policy be tested through real environment interactions, or that data be collected under a particular policy. Instead, the agent only has access to a fixed dataset $\mathcal{D}$ collected through actions taken according to some unknown \emph{behavior} policy $\pi_{\text{b}}$. 
The main challenge in this setting is that data may only span a small subset of the possible state-action pairs. Worst yet, the agent cannot observe the effects of novel out-of-distribution (OOD) state-action combinations that, by definition, are not present in $\mathcal{D}$.

A key challenge stems from the inherent uncertainty when learning from limited data \citep{levineOfflineReinforcementLearning2020,kumarStabilizingOffPolicyQLearning2019}.
Failure to account for this can lead to wild extrapolation  \citep{fujimotoOffPolicyDeepReinforcement2019, kumarConservativeQLearningOffline2020} and over/under-estimation bias in value estimates \citep{ThrunSchwartz1993, DoubleQHasselt, vanhasselt2015deep, lan2019maxmin}. 
This is a systemic problem that is exacerbated for out-of-distribution (OOD) state-actions where data is scarce.  Standard temporal difference updates to $Q$-values rely on the Bellman optimality operator which implies upwardly-extrapolated estimates tend to dominate these updates. As $Q$-values are updated with overestimated targets, they become upwardly biased even for state-actions  well-represented in $\mathcal{D}$. In turn, this can further increase the upper limit of the extrapolation errors at OOD state-actions, which forms a vicious cycle of extrapolation-inflated overestimation (\emph{extra-overestimation} for short) shown in Figure \ref{fig:extrapolate_main}.
This extra-overestimation is much more severe than the usual overestimation bias encountered in online RL \cite{ThrunSchwartz1993,DoubleQHasselt}.
As such, we critically need to constrain value estimates whenever they lead to situations that look potentially 'too good to be true', in particular when they occur where a policy might exploit them.

Likewise, naive exploration can lead to policies that diverge significantly from $\pi_b$. This, in turn, leads to even greater estimation error since we have very little data in this un(der)-explored space. Note that this is not a reason for particular concern in online RL: after all, once we are done exploring a region of the space that turns out to be less promising than we thought, we simply update the value function and stop visiting or visit rarely. Not so in batch RL where we 
cannot adjust our policy based on observing its actual effects in the environment.
These issues are exacerbated for applications with a large number of possible states and actions, such as the continuous 
settings considered in this work. Since there is no opportunity to try out a proposed policy in batch RL, learning must remain appropriately conservative for the policy to have reasonable effects when it is later actually deployed. Standard regularization techniques are leveraged in supervised learning to address such ill-specified estimation problems, and have been employed in the RL setting as well~\citep{williams1991function,schulman2015trust, fakoorp3o}.

This paper adapts standard off-policy actor-critic RL to the batch setting by adding a simple pair of regularizers. In particular, our main contribution is to introduce two novel batch-RL regularizers: The first regularizer combats the extra-overestimation bias in regions that are out-of-distribution. The second regularizer is designed to hedge against the adverse effects of policy updates that severly diverge from $\pi_b(a|s)$. The resultant method, \emph{\algpartname{}}  (\algname) exhibits state-of-the-art performance across 32 continuous control tasks from the D4RL benchmark~\cite{fu2020d4rl} demonstrating the usefulness of our regularizers for batch RL.

\section{Background}\label{sec:back}
 \vspace*{-0.5em}
Consider an infinite-horizon Markov Decision Process (MDP)~\cite{PutermanMDP1994}, $(S, A, T, r, \mu_0, \gamma)$. Here $S$ is the state space, $A \subset \mathbb{R}^d $ is a (continuous) action space, $T : S \times A \times  S \rightarrow \mathbb{R}_+$ encodes transition probabilities of the MDP, $\mu_0$ denotes the initial state distribution, $r(s,a)$ is the instantaneous reward obtained by taking action $a \in A $ in state $s \in S$, and $\gamma \in [0, 1]$ is a \emph{discount factor} for future rewards.

Given a stochastic policy $\pi(a|s)$, the sum of discounted rewards generated by taking a series of  actions $a_t \sim \pi(\cdot|s_t)$ corresponds to the \emph{return}  $R^\pi_t = \sum_{i=t}^\infty \gamma^{i-t} r(s_i,a_i)$ achieved under policy $\pi$.
The \emph{action-value function} (Q-value for short) corresponding to $\pi$, $Q^\pi(s,a)$, is defined as the expected return starting at state $s$, taking $a$, and acting according to $\pi$ thereafter,  $Q^\pi(s,a) =  \EE_{s_t \sim T, a_t\sim \pi} \left[ \sum_{t=0}^\infty \gamma^{t} r_t \mid (s_0, a_0) = (s, a) \right]$. $Q^\pi(s,a)$ obeys the Bellman equation~\cite{bellamn1957}:
\begin{align}
    Q^\pi(s,a) & = r(s, a) + \gamma \EE_{s' \sim T (\cdot|s,a), a'\sim \pi(\cdot|s')} \left[ Q^\pi(s', a') \right]
    \label{eq:realq}
\end{align}
Unlike in online RL, no interactions with the environment is allowed here, so the agent does not have the luxury of exploration. 
$\mathcal{D}$ is previously collected via actions taken according to some unknown \emph{behavior policy}  $\pi_b(a | s)$. In this work, we assume $\mathcal{D}$ consists of 1-step transition: $\{ (s_i, a_i, r_i, s'_i) \}_{i=1}^n$ where no further sample collection is permitted. 
In particular, our method, like~\cite{wuBehaviorRegularizedOffline2019, wang2020critic}, only needs a dataset consisting of a single-step transitions and does not require complete episode trajectories.
This is valuable, for instance, whenever data privacy and sharing restrictions prevent the use of the latter \citep{Lange2012}. It is also useful when combining data from sources where the interaction is still in progress, e.g.\ from ongoing user interactions.

We aim to learn an \emph{optimal} policy $\pi^*$ that maximizes the expected return, denoting the corresponding Q-values for this policy as $Q^* = Q^{\pi^*}$.  $Q^*$ is the fixed point of the Bellman \emph{optimality} operator \cite{bellamn1957}: $\mathcal{T} Q^{*}(s,a) = r(s,a) + \gamma \EE_{s' \sim T(\cdot|s,a)} \left[\max_{a'} Q^{*}(s',a') \right] $.  One way to learn $\pi^*$ is via actor-critic methods~\cite{ACTsitsiklis1999}, with policy $\pi_\phi$ and Q-value $Q_\theta$, parametrized by $\phi$ and $\theta$ respectively.

Learning good policies becomes far more difficult in batch RL as it depends on the quality/quantity of available data. Moreover, for continuous control the set of possible actions is infinite, making it nontrivial to find the optimal action even for online RL.
One option is to approximate the maximization above by only considering finitely many actions sampled from some $\pi$. This leads to the Expected Max-Q (EMaQ) operator of \citet{ghasemipour2021emaq}:
\begin{align}
  \label{eq:emaq}
  \mathcal{\overline{T}} Q(s,a)  & := r(s,a)  +\gamma \EE_{s' \sim T(\cdot|s,a)} \Big[ \max_{ \{a_k'\} } Q(s',a_k') \Big]\ .
\end{align}
Here $a_k' \sim \pi_\phi(\cdot|s')$ for $k=1,...,N$, i.e.\ the candidate actions are drawn IID from the current (stochastic) policy rather than over all possible actions.
When drawing only a single sample from $\pi_\phi$, this reduces to the standard Bellman operator (in expectation). Conversely, when $N \to \infty$ and $\pi_\phi$ has support over $A$,  this turns into the Bellman optimality operator. We learn $Q$ by minimizing the standard 1-step temporal difference (TD) error. That is, we update at iteration $t$
\begin{align}
  \label{eq:qupdate}
    \theta_{t} & \leftarrow \mathop{\mathrm{argmin}}_{\theta} \EE_{(s,a) \sim \mathcal{D}} \left[ \Big( Q_{\theta}(s, a) - \mathcal{\overline{T}} Q_{\theta_{t-1}}(s,a) \Big)^2\right]
\end{align}
Throughout, the notation $\E_{{(s,a)} \sim \mathcal{D}}$ denotes an \emph{empirical} expectation over dataset $\mathcal{D}$, whereas expectations with respect to $\pi$ are taken over the true underlying distribution corresponding to policy $\pi$.
%
Next, we update the policy by increasing the likelihood of actions with higher Q-values:
\begin{align}
  \label{eq:policyupdate}
     \phi_{t} \leftarrow \argmax_{\phi}  \EE_{ s\sim \mathcal{D} , \hat{a} \sim \pi_{\phi}(\cdot|s)} \Big[{Q}_{\theta_{t}}(s, \hat{a}) \Big]
\end{align}
using off-policy gradient-based updates \citep{silver14dpg}. 
Depending on the context, we omit $t$ from $Q_{\theta_t}$ and $\pi_{\phi_t}$.

\subsection{Extrapolation-Inflated Overestimation}\label{sec:overets}

When our Q-values are estimated via function approximation\footnote{While overestimation bias  has been mainly studied in regard to function approximation error~\cite{ThrunSchwartz1993, fujimoto18aTD3,fakoor2020ddpg,Lan2020Maxmin, vanhasselt2015deep}, \citet{DoubleQHasselt} shows that overestimation can also arise in tabular MDPs due to noise in the environment.} 
(with parameters  $\theta$), the Q-update can be erroneous and noisy~\cite{ThrunSchwartz1993}. 
Let $Q_{\theta_t}(s,a)$ denote the estimates of true underlying $Q^*(s,a)$ values at iteration $t$ of a batch RL algorithm that iterates steps (\ref{eq:qupdate}) and (\ref{eq:policyupdate}), with $\pi_{\phi_t}$ denoting the policy that maximizes $Q_{\theta_t}$. For a proper learning method, we might hope that the \emph{estimation error},  $\text{ER} := Q_{\theta_t}(s,a) - Q^*(s,a)$, has expected value $=0$ and variance $\sigma > 0$ for particular states/actions (the expectation here is over the sampling variability in the dataset $\mathcal{D}$ and stochastic updates in our batch RL algorithm). However even in this desirable scenario, Jensen's inequality nonetheless implies there will be \emph{overestimation error} $\text{OE} := \expectation [ \max_a Q_{\theta_t}(s,a) ] - \max_a Q^*(a,s) \ge 0$ 
 for the actions currently favored by $\pi_{\phi_t}$. Here the expectation is over the randomness of the underlying dataset $\mathcal{D}$ and the learning algorithm. OE will be strictly positive when the estimation errors are weakly correlated and will grow with the ER-variance  $\sigma$ \cite{ThrunSchwartz1993,lan2019maxmin}. Under the Bellman optimality or EMaQ operator, these inflated estimates are used as target values in the next Q-update in (\ref{eq:qupdate}), which thus produces a $Q_{\theta_{t+1}}(s,a)$ estimate that suffers from \emph{overestimation bias}, meaning it is expected to exceed the true Q value even if this was not the case for initial estimate $Q_{\theta_t}(s,a)$ ~\cite{ThrunSchwartz1993,fujimoto18aTD3,fakoor2020ddpg,Lan2020Maxmin, vanhasselt2015deep,DoubleQHasselt,kuznetsov20aTMC}. 

In continuous batch RL, ER may have far greater variance (larger $\sigma$) for OOD states/actions poorly represented in the dataset $\mathcal{D}$, as our function approximator $Q_{\theta_t}$ may wildly extrapolate in these data-scarce regions \cite{Lange2012,jin2020pessimism}. This in turn implies the updated policy $\pi_{\phi_{t}}$ will likely differ significantly from $\pi_b$ and favor some action $\hat{a} = \argmax_a Q_{\theta_t}(s,a)$ that is OOD  \cite{fujimotoOffPolicyDeepReinforcement2019}. The estimated value of this OOD action subsequently becomes the target in the Q-update (\ref{eq:qupdate}), and its OE will now be more severe due to the larger $\sigma$ \cite{buckman2020importance}. Even though we only apply these Q-updates to non-OOD $(s,a) \in \mathcal{D}$ whose ER may be initially smaller, the severely overestimated target values can induce increased overestimation bias in $Q_{\theta_{t+1}}(s,a)$ for  $(s,a) \in \mathcal{D}$. 
In a vicious cycle, the increase in $Q_{\theta_{t+1}}(s,a)$ for  $(s,a) \in \mathcal{D}$ can cause extrapolated $Q_{\theta_{t+1}}$ estimates to also grow for OOD actions (as there is no data to ground these OOD estimates), such that overestimation at $s,a \in \mathcal{D}$ is further amplified through additional temporal difference updates. After many iterative updates, this \emph{extra-overestimation} can eventually lead to the disturbing explosion of value estimates seen in \figref{fig:extrapolate_main}.

Several strategies address overestimation~\cite{Lan2020Maxmin,fujimoto18aTD3,fakoor2020ddpg,vanhasselt2015deep,DoubleQHasselt,kuznetsov20aTMC}. 
\citet{fujimotoOffPolicyDeepReinforcement2019} proposed a straightforward convex combination of the extremes of an estimated distribution over plausible Q values. Given a set of  estimates $Q_{\theta_j}$ for $j=1,...M$, they combine both the maximum and the minimum value for a given $(s,a)$ pair:
\begin{align}
  \label{eq:qcvx}
\overline{Q}_\theta(s, a) = \nu  \min_j Q_{\theta_j}  \hspace*{-.6mm} (s, a) + (1 - \nu)  \max_j Q_{\theta_j}  \hspace*{-.6mm} (s, a)
\end{align}
Here $\nu \in (0,1) $  determines how conservative we wish to be, and the $\min$/$\max$ are taken across $M$ Q-networks that only differ in their weight-initialization but are otherwise (independently) estimated. For larger  $\nu > 0.5$, $\overline{Q}$ may be viewed as a \emph{lower confidence bound} for $Q^*$ where the epistemic uncertainty in Q estimates is captured via an ensemble of deep Q-networks \cite{chen2017ucb}.

\section{Methods}\label{sec:approach}
Our previous discussion of extra-overestimation suggests two key sources of potential error in batch RL. Firstly, a policy learned by our algorithm might be too different from the behavior policy, which can lead to risky actions whose effects are impossible to glean from the limited data.
To address this, we propose to add an \emph{exploration-penalty} in policy updates that reduces the divergence between our learned policy $\pi_\phi$ and the policy $\pi_b$ that generated the data. Secondly, we must restrict 
overestimation in Q-values,
albeit only where it matters, that is, only when this leads to a policy exploiting overly optimistic estimates.
As such, we only need to penalize suspiciously large Q-values for actions potentially selected by our candidate policy $\pi_\phi$ (e.g.\ if their estimated  Q-value greatly exceeds the Q-value of actually observed actions).

\subsection{Q-Value Regularization}\label{sec:qvlaue}
While sequential interaction with the environment is a strong requirement that limits the practical applicability of online RL (and leads to other issues like exploration vs.\ exploitation), it has one critical benefit:
although unreliable extrapolation of Q-estimates beyond the previous observations happens during training, it is naturally corrected through further interaction with the environment. OOD state-actions with wildly overestimated values are in fact likely to be explored in subsequent updates, and their values then corrected after observing their actual effect.

In contrast,  extra-overestimation is a far more severe issue in batch RL, where we must be confident in the reliability of our learned policy before it is deployed. The issue can lead to completely useless Q-estimates. The policies corresponding to these wildly extrapolated Q-functions will perform poorly, pursuing risky actions whose true effects cannot be known based on the limited data in $\mathcal{D}$ (\figref{fig:extrapolate_main} shows an example of how extra-overestimation can lead to the disturbing explosion of Q-value estimates). 

To mitigate the key issue of extra-overestimation in $Q_\theta(s,a)$, we consider three particular aspects:\\[-1.80em]
\begin{itemize*} \setlength\itemsep{0.4em}
\item An overall shift in Q-value is less important. 
    A change from, say $Q_\theta(s,a)$ to $Q_\theta(s,a) + c(s)$ changes nothing about which
   action we might want to pick. As such, we only penalize the
  \emph{relative} shift between $Q$-values.
\item An overestimation of $Q_\theta(s,\hat{a})$ which still satisfies
  $Q_\theta(s, \hat{a}) \ll Q_\theta(s,a)$ for well-established ${a,s \in \mathcal{D}}$ will not change behavior and does not require  penalization.
\item Lastly, overestimation only matters if our policy is capable of discovering and exploiting it. 
\end{itemize*}
We use these three aspects to design a penalty for Q-value updates to be more pessimistic
\cite{buckman2020importance, jin2020pessimism}.
\begin{align}
  \label{eq:delta}
  \Delta(s,a) :=  \left[\max_{\hat{a} \in \{a_1, \ldots a_N\} \sim \pi_\phi(.|s)} \hspace*{-1.5mm} Q_{\theta}(s, \hat{a}) -  Q_{\theta}(s, a) \right]_+^{2}
\end{align}
where $s, a \in D$. We can see that the first requirement is easily satisfied, since we only compare differences $Q_\theta(s,\hat{a}) - Q_\theta(s,a)$ for different actions, given the same state $s$. The second aspect is addressed by taking the maximum between $0$ and $Q_\theta(s,\hat{a})-Q_\theta(s,a)$.
As such, we do not penalize optimism when it does not rise to the level where it would effect a change in behavior. Lastly, taking the maximum over actions drawn from the $\pi$ rather than from the maximum over all possible actions ensures that we only penalize when the overestimation would have observable consequences. As such, we limit ourselves to a rather narrow set of cases. As a result, we add this penalty to the Q-update:
\begin{equation}
\label{eq:qloss_delta}
\hspace*{-2.5mm} \theta_t \hspace*{-.3mm} \leftarrow \hspace*{-.3mm} \argmin_{\theta}   \EE_{(s,a) \sim \mathcal{D}} \hspace*{-1mm} \left[ \hspace*{-.5mm} \Big( \hspace*{-.5mm} Q_{\theta}(s, a) - \mathcal{\overline{T}} Q_{\theta_{t-1}}(s,a) \hspace*{-.5mm} \Big)^2 \hspace*{-2mm} + \eta \hspace*{-.3mm} \cdot \hspace*{-.3mm} \Delta(s,a) \hspace*{-.5mm} \right]
\end{equation}
\paragraph{Anatomy of the extra-overestimation penalty $\Delta$.}
Our proposed $\Delta$ penalty in \eqref{eq:delta} mitigates extra-overestimation bias by hindering the learned Q-value from wildly extrapolating large values for OOD state-actions. Estimated values of actions previously never seen in (known) state $s \in \mathcal{D}$ are instead encouraged to not significantly exceed the values of the actions $a$ whose effects we have seen at $s$. Note that the temporal difference update and the extra-overestimation penalty $\Delta$ in \eqref{eq:qloss_delta} are both framed on a common scale as a squared difference between two Q-functions.

How $\Delta$ affects $\theta$ becomes evident through its derivative:
\begin{equation}
 \label{eq:delta_grad}
  \boldsymbol{\nabla}_\theta \Delta(s,a) = \begin{cases}
    \Big( \boldsymbol{\nabla}_\theta Q_{\theta}(s, \hat{a}) - \boldsymbol{\nabla}_\theta Q_{\theta}(s, a)\Big) ~  \boldsymbol{\varepsilon}  & \hspace*{-1mm} \text{if } \boldsymbol{\varepsilon} > 0  \hfill \\
    0              & \hspace*{-3mm}  \text{otherwise} \\
\end{cases}
\end{equation}
Here $\hat{a} := \arg\max_{ \{\hat{a}_{k}\}_{k=1}^N} Q_{\theta}(s, \hat{a}_{k})$ again taken over $N$ actions sampled from our current policy $\pi$, and  ${\boldsymbol{\varepsilon} := Q_{\theta}(s, \hat{a}) -  Q_{\theta}(s, a)}$.

 $\Delta$ only affects certain temporal-differences where Q-values of (possibly OOD) state-actions have higher values than the $(s,a) \in \mathcal{D}$. In this case, $\Delta$ not only reduces $Q_{\theta}(s, \hat{a})$ by an amount proportional to $\boldsymbol{\varepsilon}$, but this penalty also increases the value of the previously-observed action $Q_{\theta}(s, a)$ to the same degree.
$\Delta$ thus results in a value network that favors previously observed actions.
We will generally want to choose a large conservative value of $\eta$ in applications where we know either: that the behavior policy was of high-quality (since its chosen actions should then be highly valued), or that only a tiny fraction of the possible state-action space is represented in $\mathcal{D}$, perhaps due to small sample-size  or a restricted behavior policy (since there may be severe extrapolation error). 
\vspace*{-0.1em}
\subsection{Policy Regularization}\label{sec:pc}
In batch RL, the available offline data $\mathcal{D}$ can have varying quality depending on the behavior policy $\pi_b$ used to collect the data.
Since trying out actions is not possible in batch settings, our policy network is instead updated to favor not only actions with the highest estimated Q-value but also the actions observed in $\mathcal{D}$ (whose effects we can be more certain of). Thus we introduce an \emph{exploration penalty} to regularize the policy update step:  $\phi \leftarrow \argmax_{\phi}  \EE_{ s\sim \mathcal{D} , \hat{a} \sim \pi_\phi(\cdot|s)} \Big[{Q}_{\theta}(s, \hat{a}) \Big] - \lambda \cdot {\mathbb{D}}(\pi_b, \pi_\phi)$.
In principle, various $f$-divergences~\cite{Csiszar1968fdiv} or Integral
Probability Metrics~\cite{MullerIPM1997} could employed in ${\mathbb{D}}(\cdot,\cdot)$. In practice, we limit our choice to quantities that do not require estimating the behavior
policy $\pi_b$. This leaves us with the reverse KL-divergence and IPMs
in Hilbert Space~\cite{altun2006unifying}. If we further restrict ourselves to distances that do not require sampling from $\pi_\phi$, then only the reverse KL-divergence remains.
We thus estimate
\begin{align}
 \label{eq:KL}
  \mathrm{KL}(\pi_b,\pi_\phi)  = \ & \mathbb{E}_{a \sim \pi_b(\cdot|s)}[\log \pi_b(a|s)] - \mathbb{E}_{a \sim \pi_b(\cdot|s)}[\log \pi_\phi(a|s)]  \\
   \propto  - & \mathbb{E}_{a \sim \pi_b(\cdot|s)}[\log \pi_\phi(a|s)] \approx - \frac{1}{m} \sum_{i=1}^m \log \pi_\phi(a_i|s) \label{eq:nopib}
\end{align}
whenever $a_i \sim \pi_b(\cdot|s)$. This is exactly what happens in batch RL where we have plenty of data drawn from the behavior policy, albeit no access to its explicit functional form.
Note the first entropy term in \eqref{eq:KL} can be ignored when we aim to minimize the estimated KL in terms of 
$\pi_\phi$ (as will be done in our exploration penalty).
Using \eqref{eq:nopib}, we can efficiently minimize an estimated reverse KL divergence without having to know/estimate $\pi_b$ or sample from $\pi_\phi$.

\begin{lemma} 
\label{lem:kldirections}
$\displaystyle \argmax_{\pi_\phi} \ \mathbb{E}_{s, a \sim {\pi_\phi}}[Q_\theta(s,a)] - \lambda \cdot \mathbb{E}_{s} \big[{\mathbb{D}} \big(\pi_\phi(\cdot|s),\pi_b(\cdot|s) \big) \big]$ is given by 
  \begin{align*}
  & \pi_\phi(s|a) = \frac{\pi_b(a|s)}{Z} \exp\Big(\frac{Q_\theta(s,a)}{\lambda}\Big) 
  \quad \text{ if \ ${\mathbb{D}}$  is the forward KL divergence} =
  \mathrm{KL}(\pi_\phi(\cdot|s),\pi_b(\cdot|s))
  \\
  &  \pi_\phi(s|a) = \frac{\pi_b(a|s)}{Z - Q_\theta(s,a)/\lambda}
  \quad \text{ if \ ${\mathbb{D}}$  is the reverse KL divergence} =
  \mathrm{KL}(\pi_b(\cdot|s),\pi_\phi(\cdot|s))
  \end{align*}
\end{lemma} 
where $Z \in \mathbb{R}$ is a normalizing constant in each case. Lemma \ref{lem:kldirections} shows that using  either forward 
or reverse KL-divergence as an objective, 
we recover $\pi_\phi = \pi_b$ in the limit of $\lambda \to \infty$. This is to be
expected. After all, in this case we use the distance in distributions
(thus policies) as our only criterion, and we prefer reverse KL to avoid having to estimate $\pi_b$.
\algname{} thus employs the following  policy-update (where the reverse KL is expressed as a log-likelihood as in \eqref{eq:nopib})
\begin{equation}\label{eq:piloss}
\phi \leftarrow \argmax_{\phi}  \EE_{ s\sim \mathcal{D} , \hat{a} \sim \pi_\phi(\cdot|s)} \Big[{Q}_{\theta}(s, \hat{a}) \Big] +  \lambda \cdot \EE_{(s,a) \sim \mathcal{D}} \big [ \log \pi_\phi(a|s)  \Big ]
\end{equation}
The exploration penalty helps ensure our learned $\pi_\phi$ is not significantly worse than $\pi_b$, which is far from  guaranteed in batch settings without ever testing an action.
If the data were collected by a fairly random (subpar) behavior policy, then this penalty (in expectation) acts similarly to a maximum-entropy term. The addition of such terms to similar policy-objectives has been shown to boost performance in RL methods like \emph{soft actor-critic} \cite{haarnoja2018soft}.

Note that our penalization of exploration stands in direct contrast to online RL methods that specifically incentivize exploration \cite{bellemare2017count, bellemare2016unifying}. In the batch RL, exploration is extremely dangerous as it will only take place during deployment when a policy is no longer being updated in response to the effect of its actions.
Constraining policy-updates around an existing data-generating policy has also been demonstrated as a reliable way to at least obtain an improved policy in both batch \cite{wuBehaviorRegularizedOffline2019, fujimotoOffPolicyDeepReinforcement2019} and online \cite{schulman2017proximal} settings.
Even moderate policy-improvement can often be extremely valuable (the optimal policy may be too much ask for with data of limited size or coverage of the possible state-actions). \emph{Reliable} improvement is crucial in batch settings as we cannot first test out our new policy.

\begin{remark}[Behavioral cloning occurs as $\lambda \rightarrow \infty$]
\vspace*{-0.5em}
Regularized policy updates with strong regularization (large $\lambda$)
is in the limit imitation learning.  In fact, this is the well-known
  likelihood based behavioral cloning algorithm used by
  \cite{pomerleau1991efficient}.
 \vspace*{-1.0em}
\end{remark}
If the original behavior policy $\pi_b^*$ was optimal (e.g.\ demonstration by a human-expert), then behavioral cloning should be utilized for learning from $\mathcal{D}$ \cite{osa2018algorithmic}.
However in practice, data are often collected from a subpar policy that we wish to improve upon via batch RL rather than simple imitation learning.

\subsection{\algname{} Algorithm}
\begin{wrapfigure}{r}{0.5\columnwidth}
 \vskip -3.0em
\resizebox{0.5\columnwidth}{!}{
\begin{minipage}{1.3\linewidth}
\begin{algorithm}[H]
\caption{Continuous Doubly Constrained Batch RL}
\label{alg:ours}
\begin{algorithmic}[1]
    \STATE Initialize policy $\pi_{\phi}$ and Qs: $\{Q_{\theta_j} \}_{j=1}^{M}$
    \STATE Initialize Target Qs:  $\{Q_{\theta'_j}: \theta'_j \leftarrow \theta_j\}_{j=1}^{M}$
    \FOR{$t$ in \{1, \dots, T\}}
        \STATE Sample  $\mathcal{B} = \{(s, a, r, s')\} \sim \mathcal{D}$ \\
        \STATE For each $s, s' \in \mathcal{B}$:
        sample $N$ actions $\{\hat{a}_{k} \}_{k=1}^N \hspace*{-0.6mm} \sim \hspace*{-0.6mm} \pi_{\phi}(\cdot|s), \{{a}'_{k}\}_{k=1}^N \hspace*{-0.5mm} \sim\ $ $\pi_{\phi}(\cdot|s')$ \\[0.5em]
        \STATE \textbf{$Q_{\theta}$- value update:} \\ \alglinelabel{step:value}
        \vspace*{-.2in}
            \begin{align*}
                &y(s', r) := r + \gamma \max_{{a}'_{k}}  \Big[\overline{Q}_{\theta'}(s', {a}'_{k}) \Big] \ \
                 \text{($\overline{Q}$ given by Eq~\ref{eq:qcvx})} \\
                & {\textcolor{color_alg}{ \Delta_j(s,a) :=  \Big(  \Big[  \max_{\hat{a}_{k}}  Q_{\theta_j}(s, \hat{a}_{k}) -  Q_{\theta_j}(s, a) \Big]_+ \Big)^2 } }
                \\
                & \theta_j \leftarrow \hspace*{0mm} \argmin_{\theta_j} \hspace*{-2mm} \sum_{(s, a, s') \in \mathcal{B}} \Big[ \Big( Q_{\theta_j}(s, a) - y(s', r) \Big )^2  \\[-0.78em]
                & \hspace*{30mm}
                 {\textcolor{color_alg}{ \ + \ \eta \cdot \Delta_j(s,a) }} \Big]
                \ \text{ for } j = 1,...,M
                \\[-2em]
            \end{align*}
        \STATE \textbf{$\pi_{\phi}$ - policy update: } \alglinelabel{step:policy}
        \\[-1.7em]
        \begin{align*}
        \hspace*{0mm} \phi \leftarrow \argmax_{\phi} \hspace*{-2mm}
        & \sum_{(s,a) \in \mathcal{B}, \hat{a} \sim \pi_\phi(\cdot|s)}
        \hspace*{-7mm}
        \Big[
        \overline{Q}_{\theta}(s,  \hat{a})
           { \textcolor{color_alg}{ \ + \ \lambda \cdot \log \pi_\phi(a|s) }}  \Big]
        \end{align*}
        \STATE \textbf{Update Target Networks: }
              \\[0.5em]
              $~~~~~~~\theta'_j \leftarrow \tau \theta_j + (1 - \tau)\theta'_j~~\forall j \in M $
              \\[-0.5em]
    \ENDFOR
\end{algorithmic}
\label{algo:cdc}
\end{algorithm}
\end{minipage}
}
\vskip -1.2em
\end{wrapfigure}
Furnished with the tools for Q-value and policy regularization proposed in previous sections, we introduce \algname{} in \algref{algo:cdc}. \algname{} utilizes an actor-critic framework~\cite{ACTsitsiklis1999} for continuous actions with stochastic policy $\pi_\phi$ and Q-value $Q_\theta$, parameterized by $\phi$ and $\theta$ respectively.
Our major additions to that 
$\Delta$ penalty that mitigates overestimation bias by reducing wild extrapolation in value estimates and the \emph{exploration penalty} ($\log \pi_\phi$) that discourages the estimated policy from straying to OOD state-actions very different from those whose effects we have  observed in $\mathcal{D}$.

Although the particular form of \algname{} presented in \algref{algo:cdc} optimizes a stochastic policy with the off-policy updates of \cite{silver14dpg} and temporal difference value-updates using \eqref{eq:qupdate}, we emphasize that the general idea behind \algname{} can be utilized with other forms of actor-critic updates such as those considered by \cite{haarnoja2018soft, fakoorp3o, fujimoto18aTD3}. In practice, \algname{} estimates expectations of quantities introduced throughout  via mini-batch estimates derived from samples taken from $\mathcal{D}$, and  each optimization is performed via a few stochastic gradient method iterates. 

To account for epistemic uncertainty due to the limited data, the value update in Step \ref{step:value} of \algref{alg:ours} uses  $\overline{Q}_\theta$ from \eqref{eq:qcvx} in place of $Q_\theta$.  In \algname{}, we can simply utilize the same moderately conservative value of $\nu = 0.75$ used by \cite{fujimotoOffPolicyDeepReinforcement2019}, since we are
 not purely relying on the lower confidence bound $\overline{Q}_\theta$ to correct all overestimation. For this reason, \algname{} is able to achieve  strong  performance with a small ensemble of $M=4$ Q-networks (used throughout this work), whereas \cite{ghasemipour2021emaq} require larger ensembles of 16 Q-networks and an extremely conservative $\nu=1$ in order to achieve good performance.
 
To correct extra-overestimation within each of the $M$ individual Q-networks,  \algref{alg:ours} actually applies a separate extra-overestimation penalty $\Delta_j$ specific to each Q-network. The steps of our proposed \algname{} method are detailed in \algref{algo:cdc}. In \textcolor{color_alg}{blue}, we highlight the only modifications \algname{} makes to a standard off-policy actor-critic framework that has been suitably adapted for continuous batch RL via the aforementioned techniques like EMaQ \cite{ghasemipour2021emaq} and lower-confidence bounds for Q-values  \cite{fujimotoOffPolicyDeepReinforcement2019}. Throughout, we use $\eta=0~\&~ \lambda=0$ to refer to this baseline framework (without our proposed penalties), and note that majority of modern batch RL methods like CQL \cite{kumarConservativeQLearningOffline2020}, BCQ \cite{fujimotoOffPolicyDeepReinforcement2019}, BEAR~\cite{kumarStabilizingOffPolicyQLearning2019}, BRAC~\cite{wuBehaviorRegularizedOffline2019} are built upon similar frameworks.

Although each of our proposed regularizers can be used independently and their implementation is modular, we emphasize that they complement each other: the Q-Value regularization mitigates extra-overestimation error while the policy regularizer ensures candidate policies do not stray too far from the offline data. Ablation studies show that the best performance is only achieved through simultaneous use of both regularizers  (\figref{fig:abl_all}, Table~\ref{tab:ablation}). Note that CDC is quite simple to implement: each penalty can be added to existing actor-critic RL frameworks with minimal extra code and the addition of both penalties involves no further complexity beyond the sum of the parts.

\begin{theorem} 
\vspace*{-0.25em}
For $\overline{Q}_{\theta}$ in (\ref{eq:qcvx}), let   $\mathcal{T}_{\text{CDC}}: \overline{Q}_{\theta_t} \rightarrow \overline{Q}_{\theta_{t+1}}$ denote the operator corresponding to the $\overline{Q}_{\theta}$-updates resulting from the $t^{\text{th}}$ iteration of Steps \ref{step:value}-\ref{step:policy} of a\algref{alg:ours}. $\mathcal{T}_{\text{CDC}}$ is a $L_\infty$ contraction under standard conditions that suffice for the ordinary Bellman operator to be contractive \cite{bertsekas2004stochastic, busoniu2010reinforcement, szepesvari2001efficient, antos2007fitted}.
\label{thm:contract}
\vspace*{-0.25em}
\end{theorem}

The proof and formal list of assumptions are in \appref{app:proof_contract}. Together with Banach's theorem, the contraction property established in \thmref{thm:contract} above guarantees that our \algname{} updates converge to a fixed point under commonly-assumed conditions that suffice for standard RL algorithms to converge \cite{lagoudakis2003least}.
Due to issues of (nonconvex) function approximation, it is difficult to guarantee this in practice or empirical optimality of the resulting estimates \cite{ddpg, matheron2019problem}. We do note that the addition of our two novel regularizers further enhances the contractive nature and stability of the CDC updates when $\eta, \lambda > 0$ by shrinking $Q$-values and policy action-probabilities toward the corresponding values estimated for the behavior policy (i.e.\ values computed for observations in $\mathcal{D}$). Our CDC penalties can thus not only lead to less wildly-extrapolated batch estimates, but also faster (and more stable) convergence of the learning process (as shown in \figref{fig:extrapolate_main}, where \emph{Standard actor-critic} refers to \algref{algo:cdc} where $\eta=\lambda=0$).

\begin{theorem} 
Let $\pi_\phi \in \Pi$ be the policy learned by \algname{}, $\gamma$ denote discount factor, and $n$ denote the sample size of dataset $\mathcal{D}$ generated from $\pi_b$. Also let $J(\pi)$ represent the true expected return  produced by deploying policy $\pi$ in the environment. 
Under mild assumptions listed in Appendix \ref{sec:proofs}, there exist constants $r^*, C_\lambda, V$ such that with high probability $\ge 1 - \delta$:  
\begin{align*}
    J(\pi_\phi) & \ge J(\pi_b) - \frac{ r^*}{(1-\gamma)^2} \sqrt{ C_\lambda + \sqrt{(V - \log \delta)/n} }
\end{align*}
\label{thm:reliable}
\vspace*{-1.9em}
\end{theorem}

\appref{app:proof_reliable} contains a proof and descriptions of the 
assumptions in this result.  
Theorem \ref{thm:reliable} assures us of the reliability of the policy $\pi_\phi$ produced by CDC, guaranteeing that with high probability $\pi_\phi$ will not have much worse outcomes than the behavior policy $\pi_b$, where the probability here depends on the size of the dataset $\mathcal{D}$ and our choice of policy regularization penalty $\lambda$ (the constant $C_\lambda$ is a decreasing function of $\lambda$). 
In batch settings, expecting to learn the optimal policy is futile from limited data. Even ensuring \emph{any} improvement at all over an arbitrary $\pi_b$ is ambitious when we cannot ever test any policies in the environment, and reliability of the learned $\pi_\phi$ is thus a major concern. 

\begin{theorem} 
\label{thm:oe}
Let $\textnormal{OE}_{\textnormal{ag}} = \expectation[\max_{a}Q_{\theta}(s,a)] - \max_{a} Q^{*}(s,a)$ be the overestimation error in actions favored by an agent $\textnormal{ag}$. 
Here $Q_{\theta}$ denotes the estimate of true Q-value learned by $\textnormal{ag}$, which may either use \textbf{\algname{}} (with $\eta >0$) or a \textbf{baseline} version of \algref{alg:ours} with $\eta=0$ (with the same value of $\lambda$). Under the assumptions listed in \appref{app:proof_oe}, there co-exist constants $L_1$ and $L_2$ such that
\begin{align*}
\textnormal{OE}_{\textnormal{CDC}}\leq L_{1} - \eta L_{2} \leq \textnormal{OE}_{\textnormal{\regular}} 
\end{align*}
\vspace*{-2em}
\end{theorem}
This theorem (proved in \appref{app:proof_oe}) underscores the influence of the $\eta$ parameter in terms of containing the overestimation problem in offline  Q-learning. Mitigating this overestimation, which can be done using non-zero $\eta$, can ultimately lead into better returns as we show in the experimental section. In particular, CDC achieves lower overestimation by deliberately underestimating $Q$-values for non-observed state-actions (but it limits the degree of downward bias as described in Remark \ref{rem:limited}). 
\citet{buckman2020importance,jin2020pessimism} prove that some degree of pessimism is unavoidable to ensure non-catastrophic deployment of batch RL in practice, where it is unlikely there will ever be sufficient data for the agent to accurately estimate the consequences of all possible actions in all states.

\begin{remark}[Pessimism is limited in \algname{}]
\label{rem:limited}
Extreme pessimism leads to overly conservative policies with limited returns. The degree of pessimism in \algname{} remains limited (capped  once $\Delta_j = 0$), unlike lower-confidence bounds which can become arbitrarily pessimistic and hence limited in their return.
\end{remark}

\section{Related Work} 
\vspace*{-0.5em}
Aiming for a practical framework to improve arbitrary existing policies, much research has studied batch RL~\citep{levineOfflineReinforcementLearning2020, Lange2012} and the issue of overestimation~\citep{ThrunSchwartz1993, DoubleQHasselt, vanhasselt2015deep}.  \cite{yu2020mopo, kidambi2020morel} consider model-based  approaches for batch RL, and \cite{agarwal2020optimistic} find  ensembles partly address some of the issues that arise in batch settings.
To remain suitably conservative, a popular class of approaches constrain the policy updates to remain in the vicinity of $\pi_b$ via, e.g., distributional matching~\cite{fujimotoOffPolicyDeepReinforcement2019}, support matching~\cite{kumarStabilizingOffPolicyQLearning2019,wuBehaviorRegularizedOffline2019}, imposition of a behavior-based prior \cite{siegel2020doing}, or implicit constraints via selective policy-updates
\citep{wang2020critic,pengAWR2019}. Similar to imitation learning in online setting~\cite{pomerleau1991efficient,RossGB11Imitation, HesterVPLSPSDOA17, osa2018algorithmic}, many of such methods need to explicitly estimate the behavior policy~\cite{fujimotoOffPolicyDeepReinforcement2019, kumarStabilizingOffPolicyQLearning2019, ghasemipour2021emaq}. Although methods like \citep{wang2020critic,pengAWR2019} do not have an explicit constraint on the policy update, they still can be categorized as a policy constrained-based approach as the policy update rule has been changed in a such a way that it selectively updates the policy utilizing information contained in the Q-values. Although these approaches show promising results, policy-constraint methods often work best for data collected from a high-quality (expert) behavior policy, and may struggle to significantly improve upon highly suboptimal $\pi_b$. Compared to the previous works, our \algname{} does not need to severely constrain candidate policies around $\pi_b$, which reduces achievable returns. Even with a strong policy constraint, the resulting policy is still affected by the learned Q-value, thus we still must correct Q-value issues. Instead of constraining policy updates, \cite{kumarConservativeQLearningOffline2020} advocate conservatively lower-bounding estimates of the value function. This allows for more flexibility to improve upon low-quality $\pi_b$. \cite{YaoPQI2020} considers a pessimistic and conservative approach to update Q-value by utilizing the marginalized state-action distribution of available data. Our proposed \algname{} algorithm is inspired by ideas from both the policy-constraint and value-constraint literature, demonstrating these address complementary issues of the batch RL problem and are both required in a performant solution. 

\section{Experiments}\label{sec:experiment}
In this section, we evaluate our \algname{} algorithm against existing methods on 32 tasks from the D4RL benchmark~\citep{fu2020d4rl}.
We also investigate the utility of individual \algname{} regularizers through ablation analyses, and demonstrate the broader applicability of our extra-overestimation penalty to off-policy evaluation in addition to batch RL.  Our training/evaluation setup exactly follows existing work \cite{kumarConservativeQLearningOffline2020, fu2020d4rl, fujimotoOffPolicyDeepReinforcement2019, kumarStabilizingOffPolicyQLearning2019}. See Appendices \ref{sec:app:d4rl},~\ref{sec:app:hyperparams}, and~\ref{sec:app:d4rl_detail} for a complete description of our experimental pipeline.

\paragraph{Setup. }
We compare \algname{} against existing batch RL methods: BEAR~\cite{kumarStabilizingOffPolicyQLearning2019}, BRAC-V/P~\cite{wuBehaviorRegularizedOffline2019}, BC~\cite{wuBehaviorRegularizedOffline2019}, CQL~\cite{kumarConservativeQLearningOffline2020}, BCQ~\cite{fujimotoOffPolicyDeepReinforcement2019}, 
RBVE\footnote{Comparison between \algname{} and the concurrently-proposed RBVE method  \cite{gulcehre2021regularized} is relegated to \secref{secapx:rbve-vs-cdc}.} 
\cite{gulcehre2021regularized},  and SAC~\cite{haarnoja2018soft}. This covers a rich set of strong batch RL methods ranging from behavioral cloning to value-constrained-based pessimistic methods, with the exception of SAC. SAC is a popular off-policy method that empirically performs quite well in online RL, and is included to investigate how online RL methods fare when applied to the batch setting.  Note that \algname{} was simply run on every task using the same network and the original rewards/actions provided in the task, without any manual task-specific reward-normalization/action-smoothing. Moreover, all these baseline methods also utilize an ensemble of Q networks as in~\eqref{eq:qcvx}. 

\begin{figure}[tb]
\centering\captionsetup[subfigure]{justification=centering}
\begin{subfigure}[t]{0.50\textwidth}
\centering
\includegraphics[width=\textwidth]{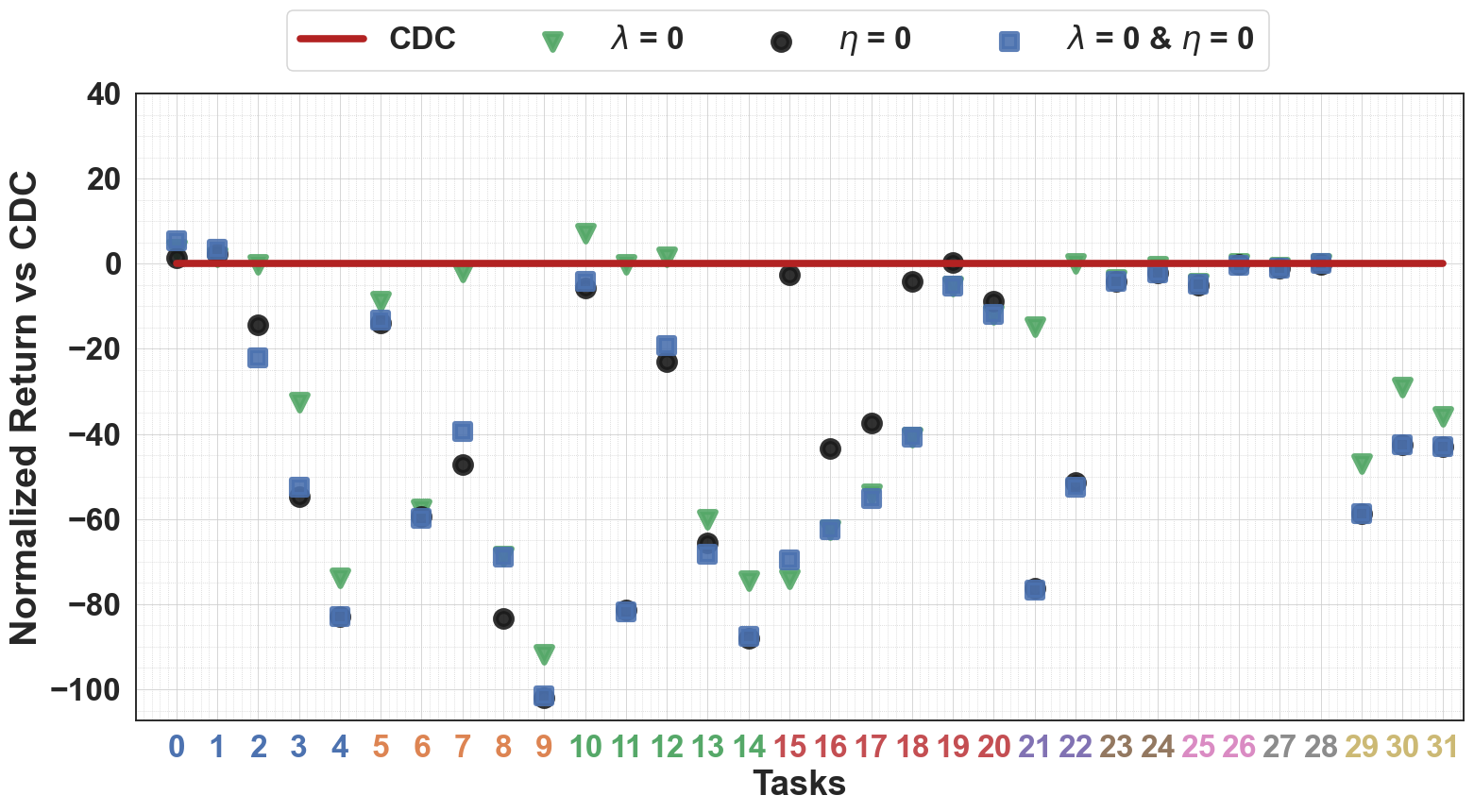}
\caption{}
\label{fig:abl_all}
\end{subfigure}%
\begin{subfigure}[t]{0.50\textwidth}
\centering
\includegraphics[width=\textwidth]{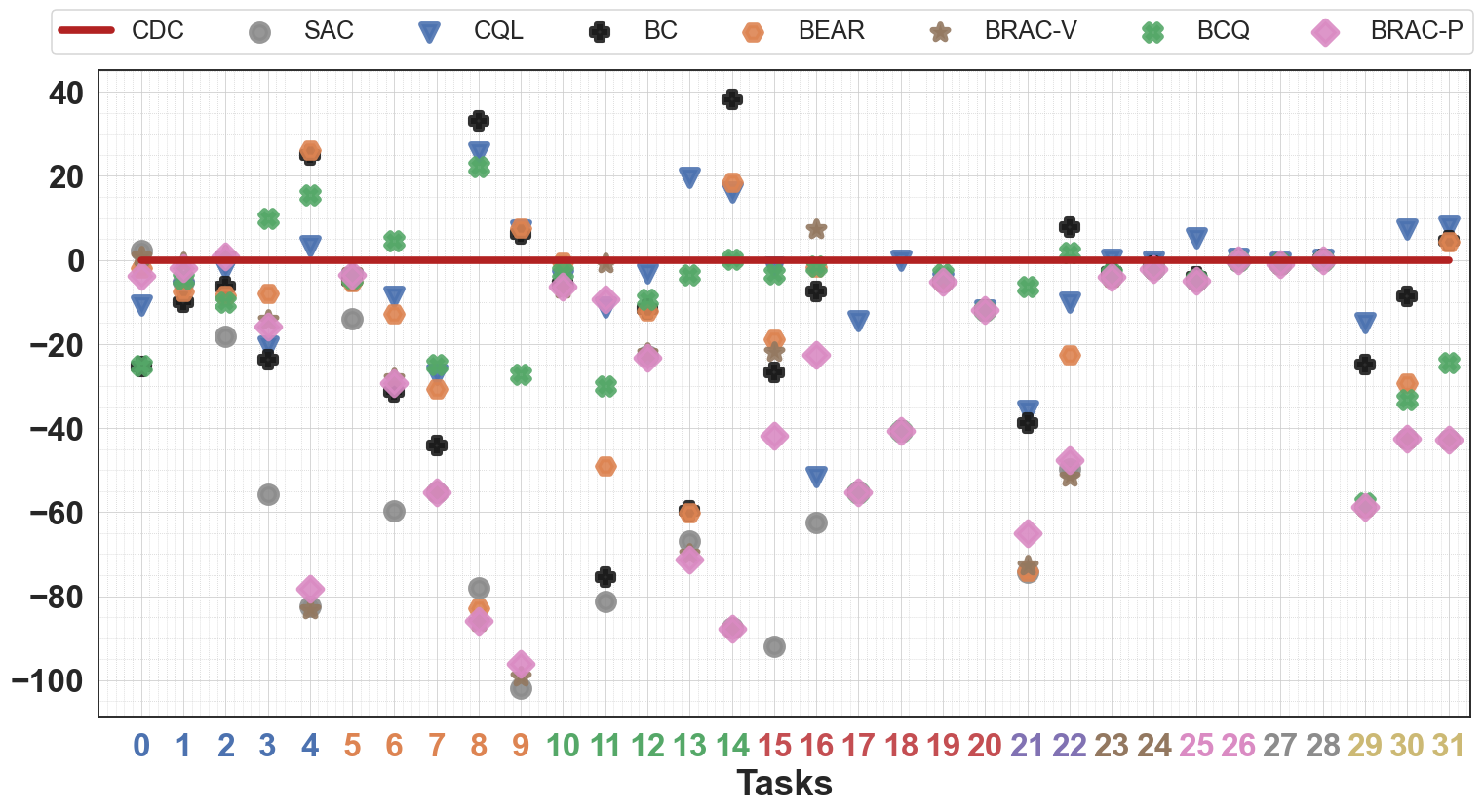}
\caption{}
\label{fig:main_all}
\end{subfigure}
\caption{
\textbf{
Difference in (normalized) 
return achieved by various algorithms vs \algname{} in 32 D4RL tasks}. X-axis colors indicate environments (see Table \ref{tab:full_d4rl}), and points below the line (\mytikzdot{}) indicate worse performance than \algname{}. 
 \textbf{\figref{fig:abl_all}} shows that 
fixing $\eta$ or $\lambda$ to zero (i.e.\ omitting our penalties) produces far worse returns than \algname{} (see also \tabref{tab:ablation}).
This ablation study proves that major performance gains for \algname{} stem from our novel pair of regularizers, as the \emph{only} difference between \algname{} and these ablated variants is either $\eta$ or $\lambda$ or both are set to zero in \algref{algo:cdc} (all other details are exactly the same). 
\textbf{\figref{fig:main_all}} compares \algname{} against existing batch RL algorithms, 
where  \algname{} overall compares favorably to each other method in head-to-head comparisons (see also \tabref{tab:full_d4rl}).
Note these figures can be compared to each other as well. 
}\end{figure}

\paragraph{Results.} 
\figref{fig:main_all} and  \tabref{tab:full_d4rl} illustrate that  \algname{} performs better than the majority of the other batch RL methods on the D4RL tasks.  Across all 32 tasks, \algname{} obtains normalized return of 1397, whereas the next-best method (CQL) achieves 1245. In head-to-head comparisons, \algname{} generates statistically significantly greater overall returns (\tabref{tab:full_d4rl}).
Unsurprisingly, behavioral-cloning (BC) works well on tasks with data generated by an expert $\pi_b$, while the online RL method, SAC, fares poorly in many tasks.
\algname{} remains reasonably competitive across all tasks, regardless of the environment or the quality of $\pi_b$ (i.e.\ random vs.\ expert).

Next we perform a comprehensive set of ablation studies to gauge the contribution of our proposed penalties in \algname{}. Here we run additional variants of \algref{alg:ours} without our penalties (i.e.\ $\eta = \lambda = 0$) 
, with only our extra-overestimation penalty ($\lambda=0$), and with only our exploration penalty ($\eta=0$). \figref{fig:abl_all} and Tables~\ref{tab:ablation} show that both penalties are critical for the strong performance of \algname{}, with the extra-overestimation penalty $\Delta$ being of greater importance than exploration (see also~\figref{fig:abl_all}).  Note that \emph{all} our ablation variants still employ the lower confidence bound from \eqref{eq:qcvx}, which alone clearly does not suffice to correct extra-overestimation.

\begin{table*}[tp]
 \centering 
\resizebox{\columnwidth}{!}{
 \begin{tabular}{c|l||c|c|c|c|c|c|c||c|c} 
 \hline 
 \textbf{Index} &  \textbf{Task Name}  & \textbf{ SAC } & \textbf{ BC } & \textbf{ BRAC-P } & \textbf{ BRAC-V } & \textbf{ BEAR } & \textbf{ BCQ } & \textbf{ CQL\footnotemark[1] } & \textbf{ $\lambda$ = 0 \& $\eta$ = 0 } & \textbf{ CDC } \\ \hline   
\textbf{\textcolor{color_0}{0}} & \small{halfcheetah-random}  & $29.6$ & $2.1$ & $23.5$ & $28.1$ & $25.5$ & $2.25$ & $16.71$ & $\mathbf{32.8}$ & $27.36$\\
\textbf{\textcolor{color_0}{1}} & \small{halfcheetah-medium}  & $40.97$ & $36.1$ & $44.0$ & $45.5$ & $38.6$ & $41.48$ & $38.97$ & $\mathbf{49.51}$ & $46.05$\\
\textbf{\textcolor{color_0}{2}} & \small{halfcheetah-medium-replay}  & $26.47$ & $38.4$ & $45.6$ & $\mathbf{45.9}$ & $36.2$ & $34.79$ & $42.77$ & $22.72$ & $44.74$\\
\textbf{\textcolor{color_0}{3}} & \small{halfcheetah-medium-expert}  & $3.78$ & $35.8$ & $43.8$ & $45.3$ & $51.7$ & $\mathbf{69.64}$ & $39.18$ & $7.12$ & $59.64$\\
\textbf{\textcolor{color_0}{4}} & \small{halfcheetah-expert}  & $-0.41$ & $107.0$ & $3.8$ & $-1.1$ & $\mathbf{108.2}$ & $97.44$ & $85.49$ & $-0.95$ & $82.05$\\
\textbf{\textcolor{color_1}{5}} & \small{hopper-random}  & $0.8$ & $9.8$ & $11.1$ & $12.0$ & $9.5$ & $10.6$ & $10.37$ & $1.58$ & $\mathbf{14.76}$\\
\textbf{\textcolor{color_1}{6}} & \small{hopper-medium}  & $0.81$ & $29.0$ & $31.2$ & $32.3$ & $47.6$ & $\mathbf{65.07}$ & $51.79$ & $0.58$ & $60.39$\\
\textbf{\textcolor{color_1}{7}} & \small{hopper-medium-replay}  & $0.59$ & $11.8$ & $0.7$ & $0.8$ & $25.3$ & $31.05$ & $28.67$ & $16.4$ & $\mathbf{55.89}$\\
\textbf{\textcolor{color_1}{8}} & \small{hopper-medium-expert}  & $8.96$ & $\mathbf{119.9}$ & $1.1$ & $0.8$ & $4.0$ & $109.1$ & $112.46$ & $18.07$ & $86.9$\\
\textbf{\textcolor{color_1}{9}} & \small{hopper-expert}  & $0.8$ & $109.0$ & $6.6$ & $3.7$ & $\mathbf{110.3}$ & $75.52$ & $109.97$ & $1.27$ & $102.75$\\
\textbf{\textcolor{color_2}{10}} & \small{walker2d-random}  & $1.3$ & $1.6$ & $0.8$ & $0.5$ & $6.7$ & $4.31$ & $2.77$ & $2.96$ & $\mathbf{7.22}$\\
\textbf{\textcolor{color_2}{11}} & \small{walker2d-medium}  & $0.81$ & $6.6$ & $72.7$ & $81.3$ & $33.2$ & $52.03$ & $71.03$ & $0.33$ & $\mathbf{82.13}$\\
\textbf{\textcolor{color_2}{12}} & \small{walker2d-medium-replay}  & $0.04$ & $11.3$ & $-0.3$ & $0.9$ & $10.8$ & $13.67$ & $19.95$ & $3.81$ & $\mathbf{22.96}$\\
\textbf{\textcolor{color_2}{13}} & \small{walker2d-medium-expert}  & $4.09$ & $11.3$ & $-0.3$ & $0.9$ & $10.8$ & $67.26$ & $\mathbf{90.55}$ & $2.65$ & $70.91$\\
\textbf{\textcolor{color_2}{14}} & \small{walker2d-expert}  & $0.05$ & $\mathbf{125.7}$ & $-0.2$ & $-0.0$ & $106.1$ & $87.59$ & $103.6$ & $-0.1$ & $87.54$\\
\textbf{\textcolor{color_3}{15}} & \small{antmaze-umaze}  & $0.0$ & $65$ & $50$ & $70$ & $73$ & $88.52$ & $90.12$ & $22.22$ & $\mathbf{91.85}$\\
\textbf{\textcolor{color_3}{16}} & \small{antmaze-umaze-diverse}  & $0.0$ & $55$ & $40$ & $\mathbf{70}$ & $61$ & $61.11$ & $11.11$ & $0.0$ & $62.59$\\
\textbf{\textcolor{color_3}{17}} & \small{antmaze-medium-play}  & $0.0$ & $0$ & $0$ & $0$ & $0$ & $0.0$ & $40.74$ & $0.0$ & $\mathbf{55.19}$\\
\textbf{\textcolor{color_3}{18}} & \small{antmaze-medium-diverse}  & $0.0$ & $0$ & $0$ & $0$ & $0$ & $0.0$ & $40.74$ & $0.0$ & $\mathbf{40.74}$\\
\textbf{\textcolor{color_3}{19}} & \small{antmaze-large-play}  & $0.0$ & $0$ & $0$ & $0$ & $0$ & $1.85$ & $0.0$ & $0.0$ & $\mathbf{5.19}$\\
\textbf{\textcolor{color_3}{20}} & \small{antmaze-large-diverse}  & $0.0$ & $0$ & $0$ & $0$ & $0$ & $0.0$ & $0.0$ & $0.0$ & $\mathbf{11.85}$\\
\textbf{\textcolor{color_4}{21}} & \small{pen-human}  & $-1.15$ & $34.4$ & $8.1$ & $0.6$ & $-1.0$ & $66.88$ & $37.5$ & $-3.43$ & $\mathbf{73.19}$\\
\textbf{\textcolor{color_4}{22}} & \small{pen-cloned}  & $-0.64$ & $\mathbf{56.9}$ & $1.6$ & $-2.5$ & $26.5$ & $50.86$ & $39.2$ & $-3.4$ & $49.18$\\
\textbf{\textcolor{color_5}{23}} & \small{hammer-human}  & $0.26$ & $1.5$ & $0.3$ & $0.2$ & $0.3$ & $0.91$ & $\mathbf{4.4}$ & $0.26$ & $4.34$\\
\textbf{\textcolor{color_5}{24}} & \small{hammer-cloned}  & $0.27$ & $0.8$ & $0.3$ & $0.3$ & $0.3$ & $0.38$ & $2.1$ & $0.26$ & $\mathbf{2.37}$\\
\textbf{\textcolor{color_6}{25}} & \small{door-human}  & $-0.34$ & $0.5$ & $-0.3$ & $-0.3$ & $-0.3$ & $-0.05$ & $\mathbf{9.9}$ & $-0.16$ & $4.62$\\
\textbf{\textcolor{color_6}{26}} & \small{door-cloned}  & $-0.34$ & $-0.1$ & $-0.1$ & $-0.1$ & $-0.1$ & $0.01$ & $\mathbf{0.4}$ & $-0.36$ & $0.01$\\
\textbf{\textcolor{color_7}{27}} & \small{relocate-human}  & $-0.31$ & $0$ & $-0.3$ & $-0.3$ & $-0.3$ & $-0.04$ & $0.2$ & $-0.31$ & $\mathbf{0.73}$\\
\textbf{\textcolor{color_7}{28}} & \small{relocate-cloned}  & $-0.11$ & $\mathbf{-0.1}$ & $-0.3$ & $-0.3$ & $-0.2$ & $-0.28$ & $-0.1$ & $-0.15$ & $-0.24$\\
\textbf{\textcolor{color_8}{29}} & \small{kitchen-complete}  & $0.0$ & $33.8$ & $0$ & $0$ & $0$ & $0.83$ & $43.8$ & $0.0$ & $\mathbf{58.7}$\\
\textbf{\textcolor{color_8}{30}} & \small{kitchen-partial}  & $0.0$ & $33.8$ & $0$ & $0$ & $13.1$ & $9.26$ & $\mathbf{49.8}$ & $0.0$ & $42.5$\\
\textbf{\textcolor{color_8}{31}} & \small{kitchen-mixed}  & $0.0$ & $47.5$ & $0$ & $0$ & $47.2$ & $18.43$ & $\mathbf{51}$ & $0.0$ & $42.87$\\
 \hline \hline 
 & Total Score & $116.28$ & $984.4$ & $383.4$ & $434.5$ & $844.0$ & $1060.46$ & $1245.2$ & $173.67$ & $\mathbf{1396.99}$ \\ 
 \hline \hline 
 & $\mathbf{p}$\textbf{-value vs.\ CDC} & 7.0e-07 & 1.6e-03 & 5.3e-07 & 3.9e-06 & 1.1e-04 & 6.1e-04 & 3.6e-02 & 1.8e-06 &  -  \\ 
 \hline
 \end{tabular} }
 \caption{{ \textbf{Return achieved in deployment of policies learned via different batch RL methods.} The return in each environment here is normalized using \eqref{eq:score_norm} as originally advocated by \cite{fu2020d4rl}.}
 For each method: we perform a head-to-head comparison against \algname{} across the D4RL tasks, reporting the $p$-value of a (one-sided) Wilcoxon signed rank test~\cite{Wilcoxon1945} that compares this method's return against that of \algname{} (over the 32 tasks). Here $\lambda = 0~\&~\eta = 0$ is variant of \algref{alg:ours} without our penalties where it proves that major performance gains for CDC stem from our novel pair of regularizers.} 
 \label{tab:full_d4rl}
 \end{table*}
\footnotetext[1]{
The results for CQL are taken from the official author-provided codes [\url{https://github.com/aviralkumar2907/CQL}] of \citesi{kumarConservativeQLearningOffline2020}. The published CQL codes are used to produce results for all but Adroit and FrankaKitchen where the codes are not available. For these latter domains, we simply use the CQL results reported in the paper of \citesi{kumarConservativeQLearningOffline2020}. }

\subsection{Offline Policy Evaluation}
\label{sec:exppolicyevaluation}
\begin{wrapfigure}{r}{0.5\textwidth}
\vskip -5em
\centering
\includegraphics[width=0.5\textwidth]{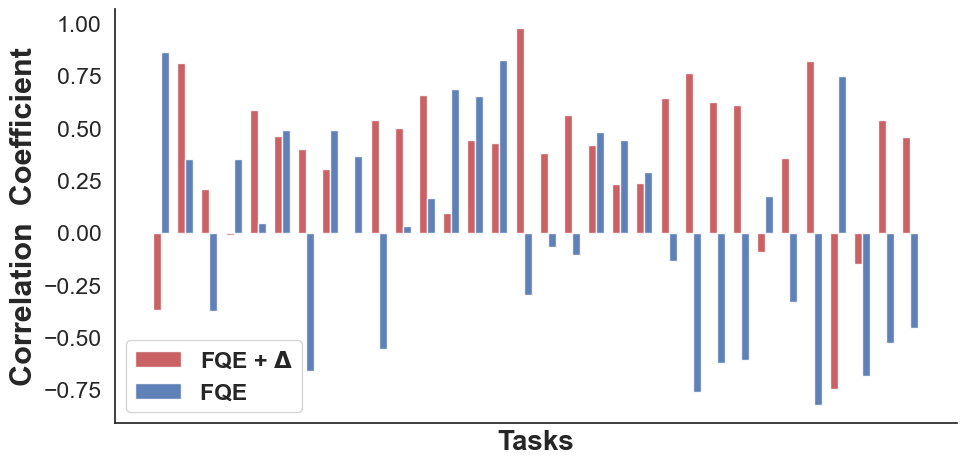}
\vspace*{-2em}
\caption{\textbf{How well OPE estimates correlate with actual return} achieved by 20 different policies for each D4RL task. Due to unmitigated overestimation, FQE estimates correlate \emph{negatively} with true returns in 15 of 32 tasks (using  $\Delta$ in FQE reduces this to 4). 
}
\label{fig:fqe_pearson}
\vskip -0.9em
\end{wrapfigure}
The true practical applicability of batch RL remains however hampered without the ability to do proper algorithm/hyperparameter selection. Table \ref{tab:full_d4rl} shows that no algorithm universally dominates all others across all  environments or behavior-policies. In practice, it is difficult to know which technique will perform best, unless one can do proper offline policy evaluation (OPE) of different candidate policies before their actual online deployment~\cite{paine2020hyperparameter, fu2021benchmarks}.

OPE aims to estimate the performance of a given policy under the same setting considered here, with offline data collected by an unknown behavior policy~\cite{paine2020hyperparameter,  Hoangfqe19}.
Beyond algorithm/hyperparameter comparison, OPE is often employed for critical policy-making decisions where environment interaction is no longer an option, e.g., sensitive healthcare applications~\cite{opegottesman20a}.
One practical OPE method for data of the form in $\mathcal{D}$ is Fitted Q Evaluation (FQE) \cite{Hoangfqe19}. To score a given policy $\pi$, FQE iterates temporal difference updates of the form \eqref{eq:qupdate} using the standard Bellman operator from \eqref{eq:realq} in place of EMaQ. After learning an estimate $\hat{Q}^\pi$, FQE simply estimates the return of $\pi$ via the expectation of $\hat{Q}^\pi(s,a)$ over the initial state distribution and actions sampled from $\pi$.

However, like batch RL, OPE also relies on limited data and thus
can still suffer from severe Q-value estimation errors.
To contain the overestimation bias, we can regularize the FQE temporal difference updates with our $\Delta$ penalty, in a similar manner to \eqref{eq:qloss_delta}. \figref{fig:fqe_pearson} compares the performance of $\Delta$-penalization of FQE (with $\eta=1$ throughout) against the standard unregularized FQE. Here we use both OPE methods to score 20 different policies (learned via different settings) and gauge OPE-quality via the Pearson correlation between OPE estimated returns and the actual return (over our $20$ policies). We observe higher correlation for FQE + $\Delta$ (0.37 on average) over FQE (0.01 on average) in the majority of tasks, demonstrating the usefulness of our regularizers. The usefulness of our regularizers thus extend beyond batch RL and carry to the off-policy evaluation setting.

\section{Discussion}
Here we propose a simple and effective algorithm for batch RL by introducing a simple pair of regularizers that abate the challenge of learning how to act from limited data. The first constrains the value update to mitigate extra-overestimation error, while the latter constrains the policy update to ensure candidate policies do not stray too far from the offline data. Unlike previous work, this paper highlights the utility of simultaneous policy and value regularization in batch RL. One can envision other combinations of alternative policy and value regularizers that may perform even better than the particular policy/value penalties used in \algname{}. That said, our \algname{} penalties are particularly simple to incorporate into arbitrary actor-critic RL frameworks, and operate synergistically as illustrated in the ablation studies. Comprehensive experiments on standard offline  continuous-control benchmarks suggest that CDC compares favorably with state-of-the-art methods for batch RL, and our proposed penalties are also useful to improve offline policy evaluation. The broader impact of this work will hopefully be to improve batch RL performance in offline applications, but we caution that unobserved confounding remains another key challenge in real-world data that was not addressed in this work.

\clearpage
\bibliography{references}
\bibliographystyle{abbrvnat}

\clearpage
\clearpage \newpage
\beginsupplement
\onecolumn
\appendix
 
\begin{center}
{\LARGE  \bf Appendix: \ 
Continuous Doubly Constrained Batch \\
\hspace*{25.5mm} Reinforcement Learning  
}
\end{center}

\section{Experiment Details}\label{sec:app:d4rl}

\begin{figure}
\centering\captionsetup[subfigure]{justification=centering}
\begin{subfigure}[t]{ 0.2\textwidth} 
\centering 
\includegraphics[width=\textwidth]{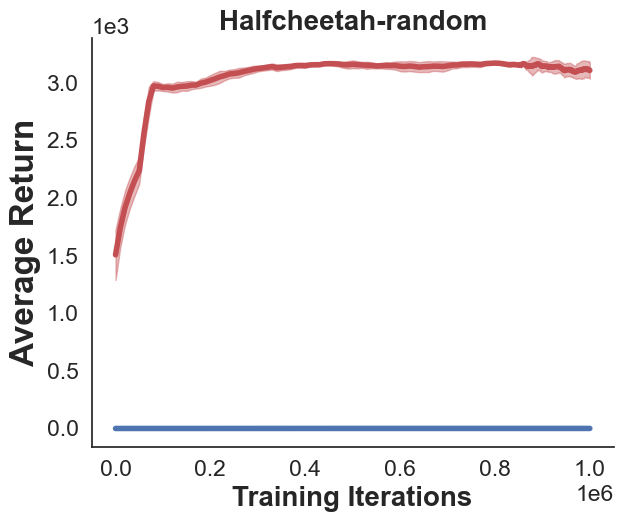} 
\label{fig:halfcheetah-random-v0_3bcqcd.png} 
\end{subfigure}%
~ 
\begin{subfigure}[t]{ 0.2\textwidth} 
\centering 
\includegraphics[width=\textwidth]{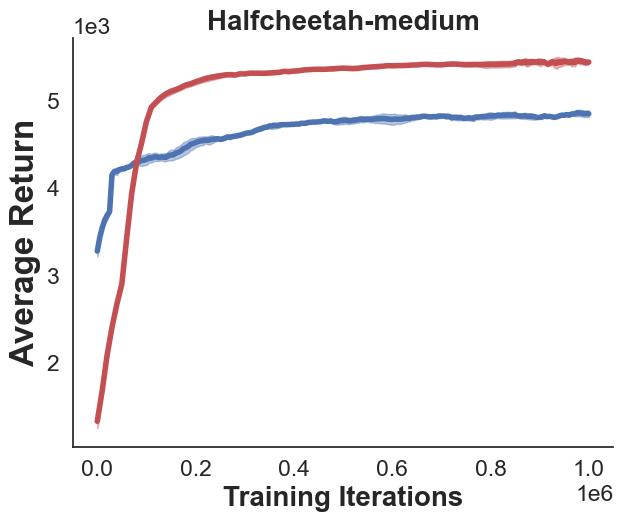} 
\label{fig:halfcheetah-medium-v0_3bcqcd.png} 
\end{subfigure}%
~ 
\begin{subfigure}[t]{ 0.2\textwidth} 
\centering 
\includegraphics[width=\textwidth]{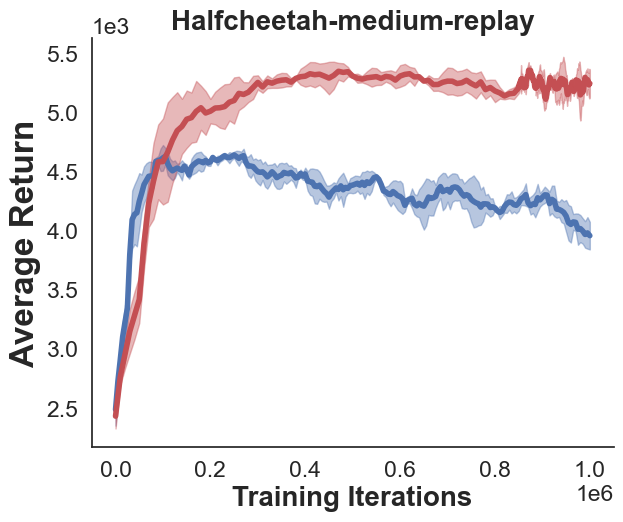} 
\label{fig:halfcheetah-medium-replay-v0_3bcqcd.png} 
\end{subfigure}%
~ 
\begin{subfigure}[t]{ 0.2\textwidth} 
\centering 
\includegraphics[width=\textwidth]{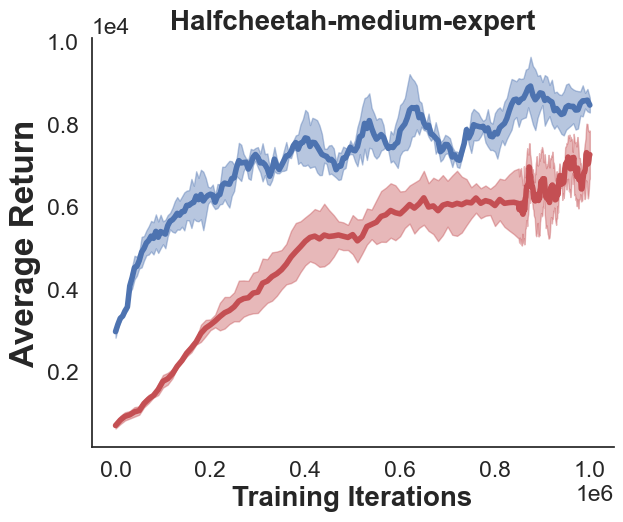} 
\label{fig:halfcheetah-medium-expert-v0_3bcqcd.png} 
\end{subfigure}%

\begin{subfigure}[t]{ 0.2\textwidth} 
\centering 
\includegraphics[width=\textwidth]{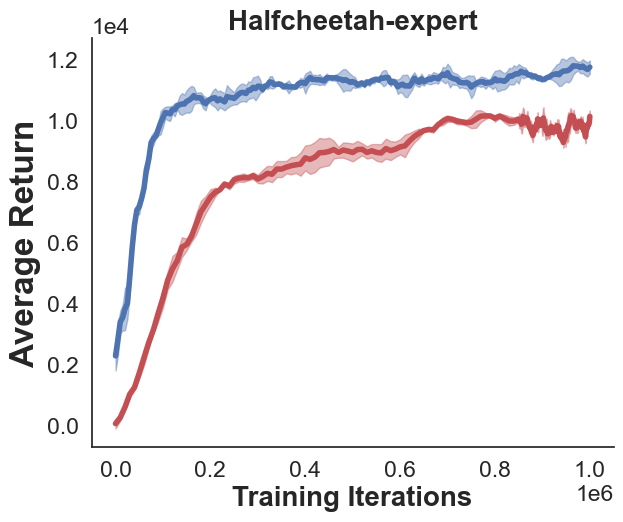} 
\label{fig:halfcheetah-expert-v0_3bcqcd.png} 
\end{subfigure}%
~ 
\begin{subfigure}[t]{ 0.2\textwidth} 
\centering 
\includegraphics[width=\textwidth]{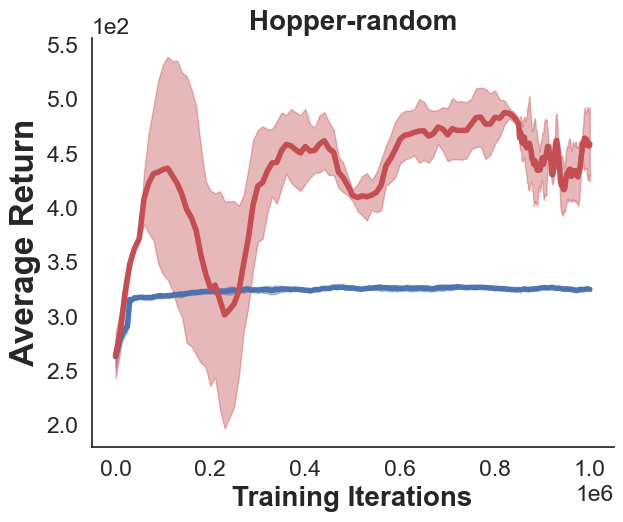} 
\label{fig:hopper-random-v0_3bcqcd.png} 
\end{subfigure}%
~ 
\begin{subfigure}[t]{ 0.2\textwidth} 
\centering 
\includegraphics[width=\textwidth]{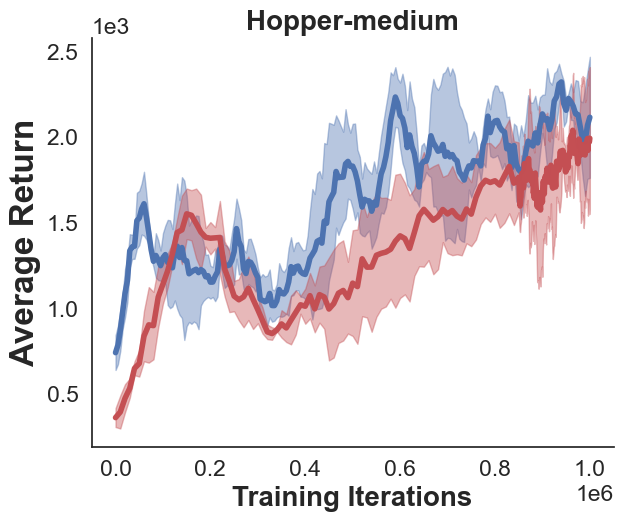} 
\label{fig:hopper-medium-v0_3bcqcd.png} 
\end{subfigure}%
~ 
\begin{subfigure}[t]{ 0.2\textwidth} 
\centering 
\includegraphics[width=\textwidth]{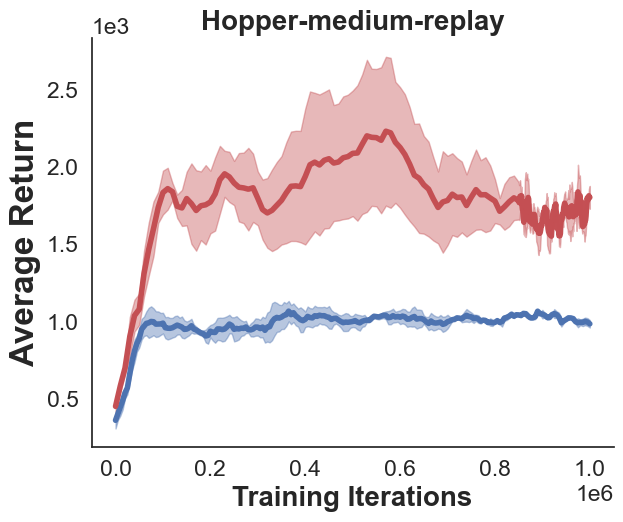} 
\label{fig:hopper-medium-replay-v0_3bcqcd.png} 
\end{subfigure}%

\begin{subfigure}[t]{ 0.2\textwidth} 
\centering 
\includegraphics[width=\textwidth]{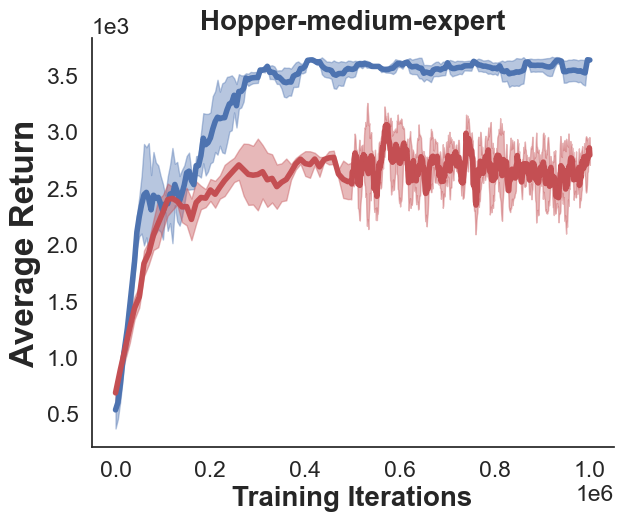} 
\label{fig:hopper-medium-expert-v0_3bcqcd.png} 
\end{subfigure}%
~ 
\begin{subfigure}[t]{ 0.2\textwidth} 
\centering 
\includegraphics[width=\textwidth]{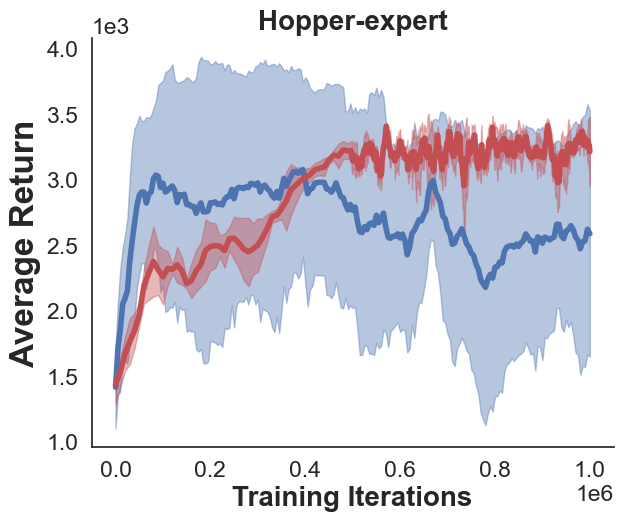} 
\label{fig:hopper-expert-v0_3bcqcd.png} 
\end{subfigure}%
~ 
\begin{subfigure}[t]{ 0.2\textwidth} 
\centering 
\includegraphics[width=\textwidth]{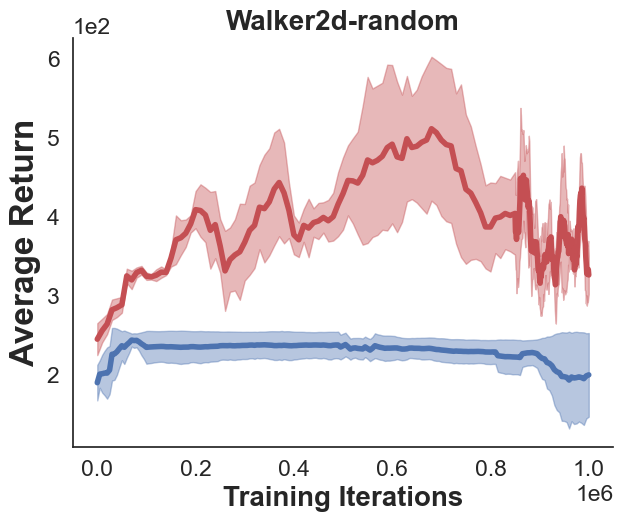} 
\label{fig:walker2d-random-v0_3bcqcd.png} 
\end{subfigure}%
~ 
\begin{subfigure}[t]{ 0.2\textwidth} 
\centering 
\includegraphics[width=\textwidth]{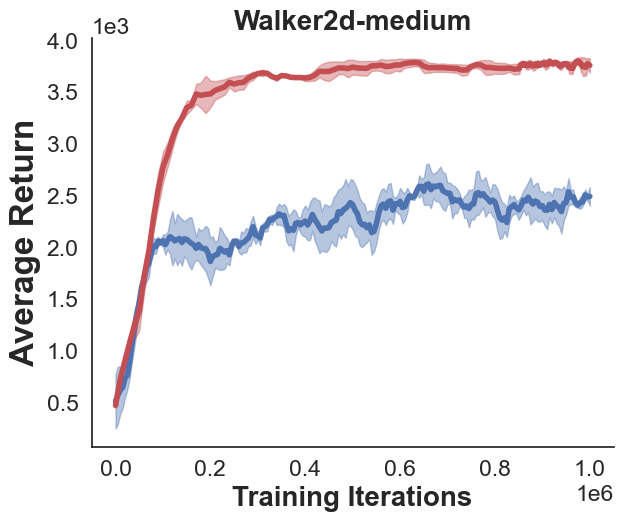} 
\label{fig:walker2d-medium-v0_3bcqcd.png} 
\end{subfigure}%

\begin{subfigure}[t]{ 0.2\textwidth} 
\centering 
\includegraphics[width=\textwidth]{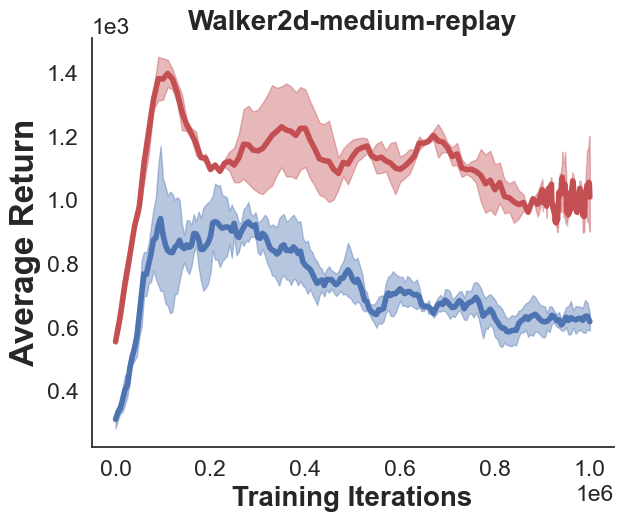} 
\label{fig:walker2d-medium-replay-v0_3bcqcd.png} 
\end{subfigure}%
~ 
\begin{subfigure}[t]{ 0.2\textwidth} 
\centering 
\includegraphics[width=\textwidth]{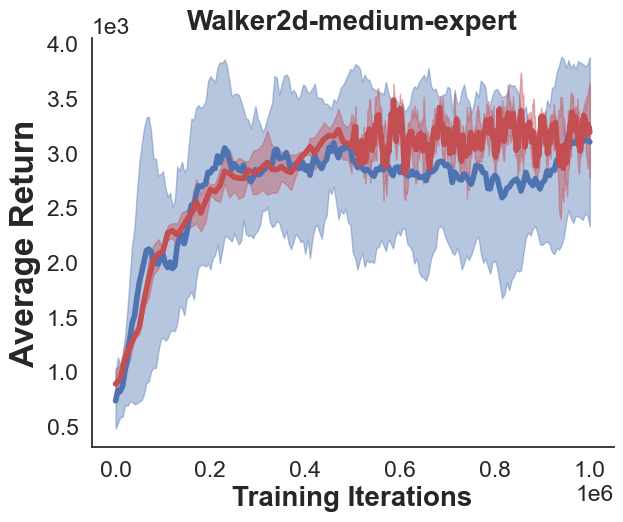} 
\label{fig:walker2d-medium-expert-v0_3bcqcd.png} 
\end{subfigure}%
~ 
\begin{subfigure}[t]{ 0.2\textwidth} 
\centering 
\includegraphics[width=\textwidth]{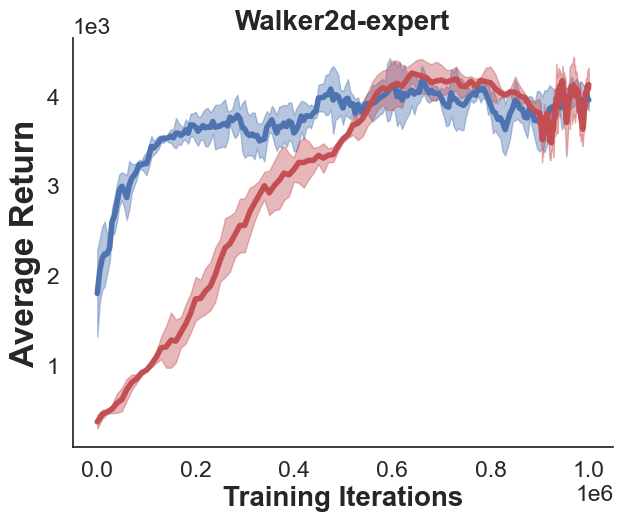} 
\label{fig:walker2d-expert-v0_3bcqcd.png} 
\end{subfigure}%
~ 
\begin{subfigure}[t]{ 0.2\textwidth} 
\centering 
\includegraphics[width=\textwidth]{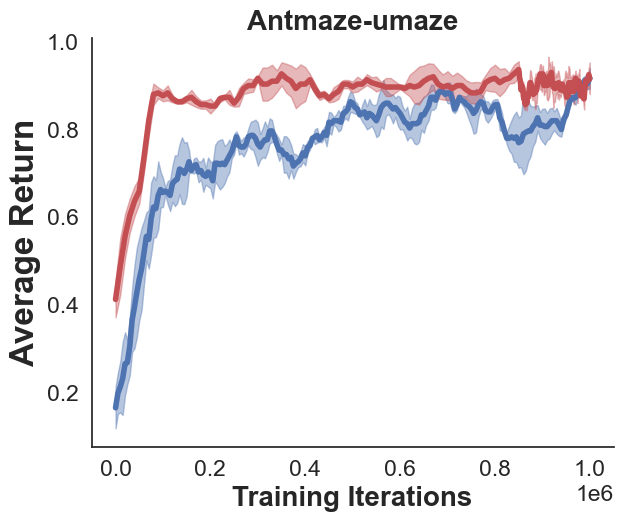} 
\label{fig:antmaze-umaze-v0_3bcqcd.png} 
\end{subfigure}%

\begin{subfigure}[t]{ 0.2\textwidth} 
\centering 
\includegraphics[width=\textwidth]{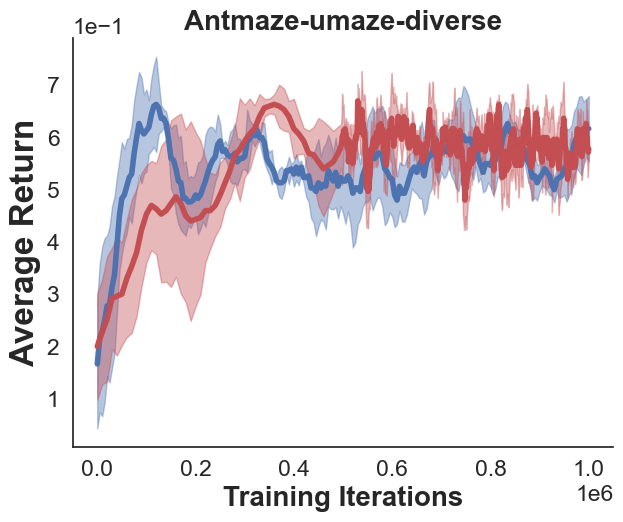} 
\label{fig:antmaze-umaze-diverse-v0_3bcqcd.png} 
\end{subfigure}%
~ 
\begin{subfigure}[t]{ 0.2\textwidth} 
\centering 
\includegraphics[width=\textwidth]{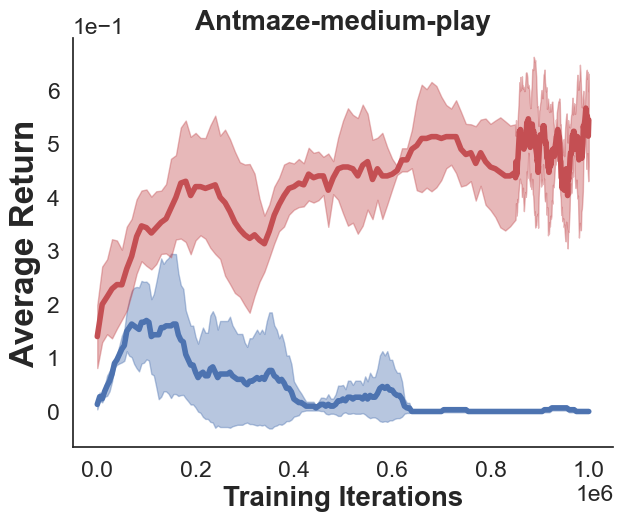} 
\label{fig:antmaze-medium-play-v0_3bcqcd.png} 
\end{subfigure}%
~ 
\begin{subfigure}[t]{ 0.2\textwidth} 
\centering 
\includegraphics[width=\textwidth]{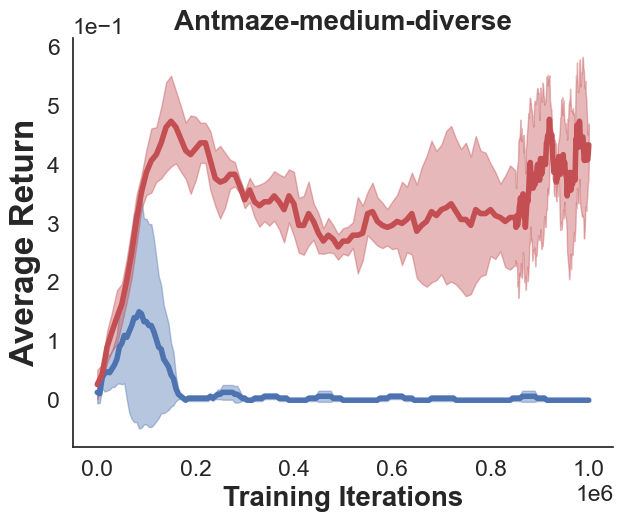} 
\label{fig:antmaze-medium-diverse-v0_3bcqcd.png} 
\end{subfigure}%
~ 
\begin{subfigure}[t]{ 0.2\textwidth} 
\centering 
\includegraphics[width=\textwidth]{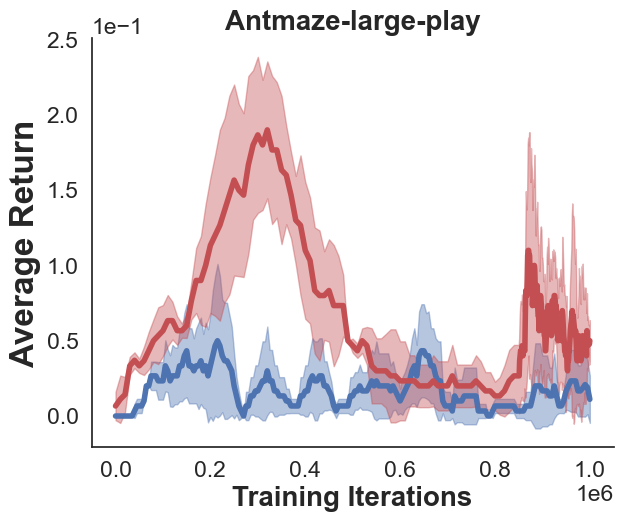} 
\label{fig:antmaze-large-play-v0_3bcqcd.png} 
\end{subfigure}%

\begin{subfigure}[t]{ 0.2\textwidth} 
\centering 
\includegraphics[width=\textwidth]{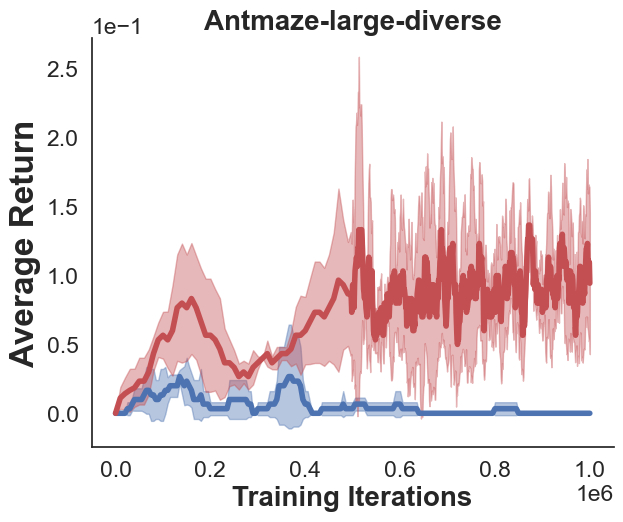} 
\label{fig:antmaze-large-diverse-v0_3bcqcd.png} 
\end{subfigure}%
~ 
\begin{subfigure}[t]{ 0.2\textwidth} 
\centering 
\includegraphics[width=\textwidth]{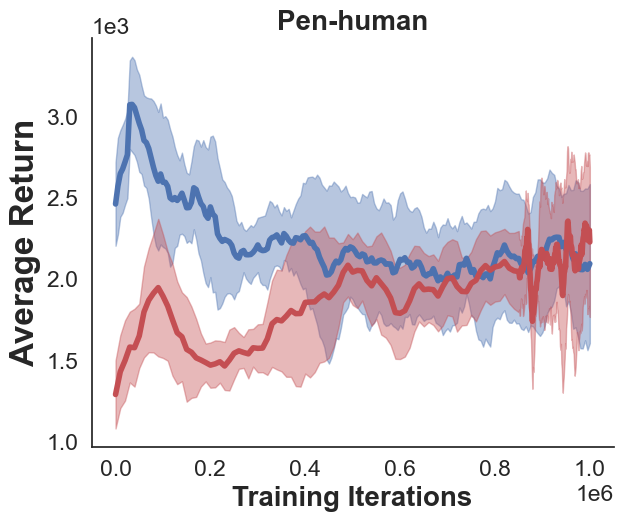} 
\label{fig:pen-human-v0_3bcqcd.png} 
\end{subfigure}%
~ 
\begin{subfigure}[t]{ 0.2\textwidth} 
\centering 
\includegraphics[width=\textwidth]{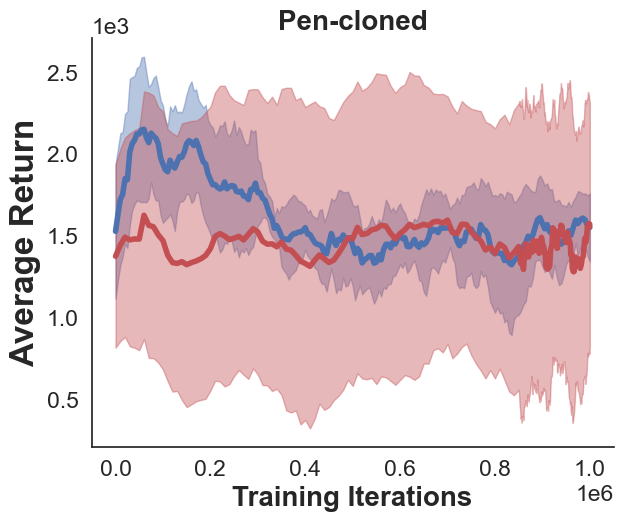} 
\label{fig:pen-cloned-v0_3bcqcd.png} 
\end{subfigure}%
~ 
\begin{subfigure}[t]{ 0.2\textwidth} 
\centering 
\includegraphics[width=\textwidth]{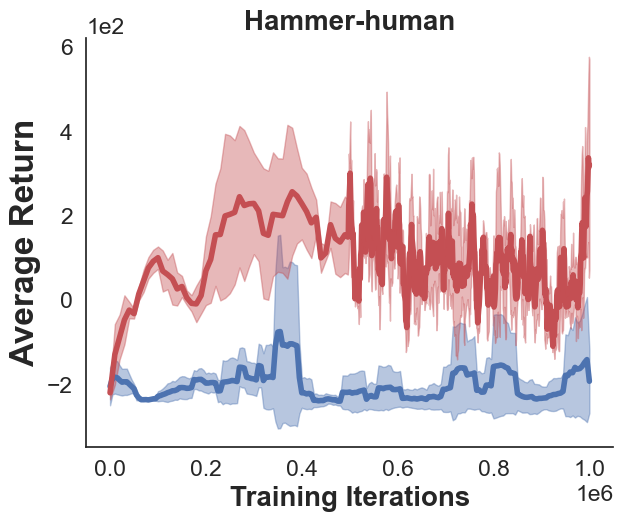} 
\label{fig:hammer-human-v0_3bcqcd.png} 
\end{subfigure}%

\begin{subfigure}[t]{ 0.2\textwidth} 
\centering 
\includegraphics[width=\textwidth]{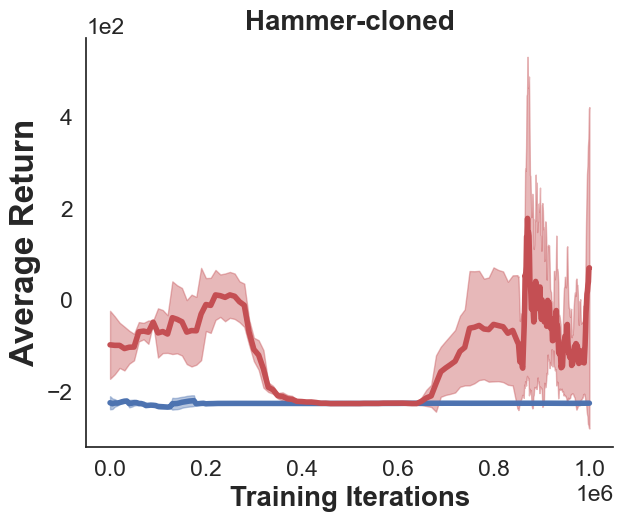} 
\label{fig:hammer-cloned-v0_3bcqcd.png} 
\end{subfigure}%
~ 
\begin{subfigure}[t]{ 0.2\textwidth} 
\centering 
\includegraphics[width=\textwidth]{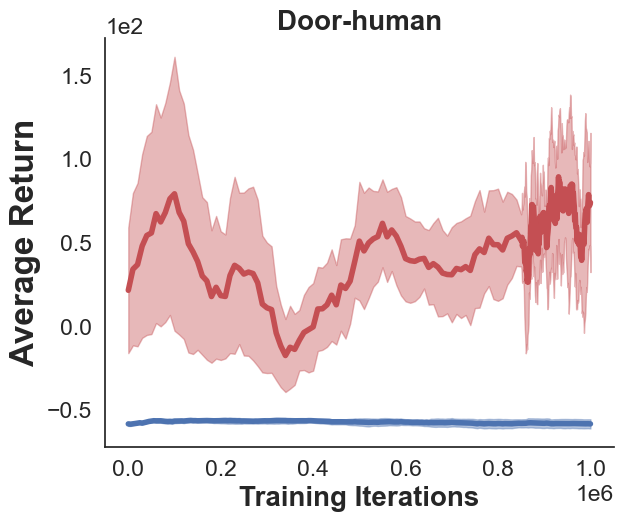} 
\label{fig:door-human-v0_3bcqcd.png} 
\end{subfigure}%
~ 
\begin{subfigure}[t]{ 0.2\textwidth} 
\centering 
\includegraphics[width=\textwidth]{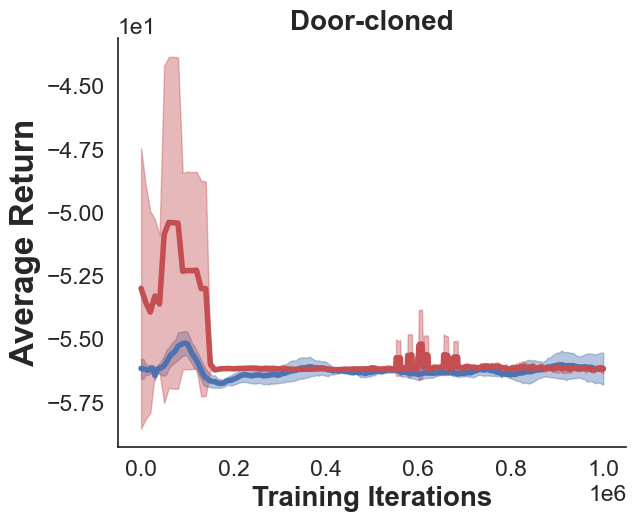} 
\label{fig:door-cloned-v0_3bcqcd.png} 
\end{subfigure}%
~ 
\begin{subfigure}[t]{ 0.2\textwidth} 
\centering 
\includegraphics[width=\textwidth]{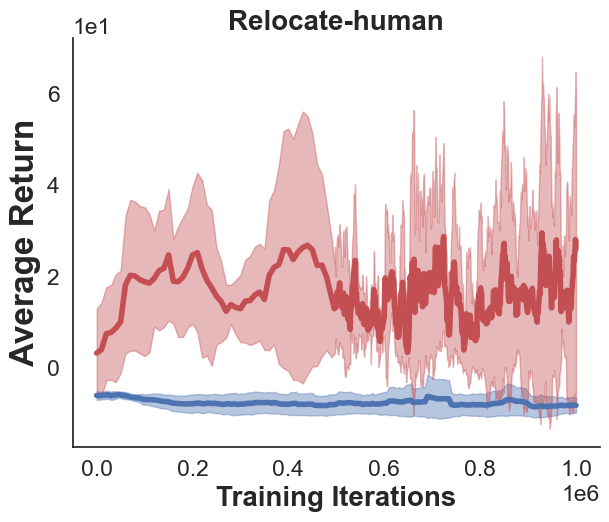} 
\label{fig:relocate-human-v0_3bcqcd.png} 
\end{subfigure}%

\begin{subfigure}[t]{ 0.2\textwidth} 
\centering 
\includegraphics[width=\textwidth]{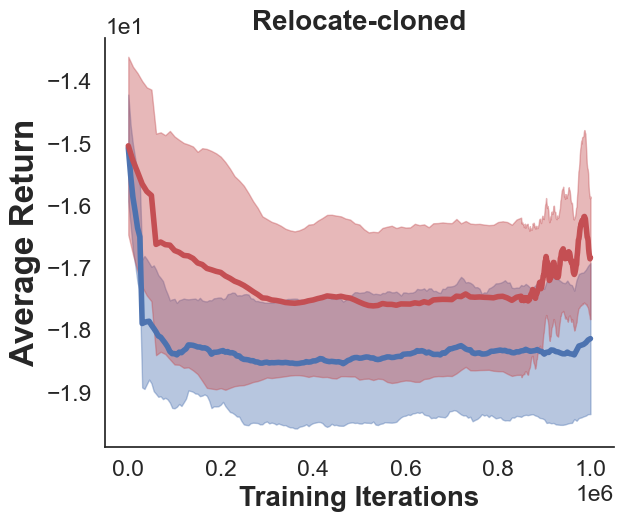} 
\label{fig:relocate-cloned-v0_3bcqcd.png} 
\end{subfigure}%
~ 
\begin{subfigure}[t]{ 0.2\textwidth} 
\centering 
\includegraphics[width=\textwidth]{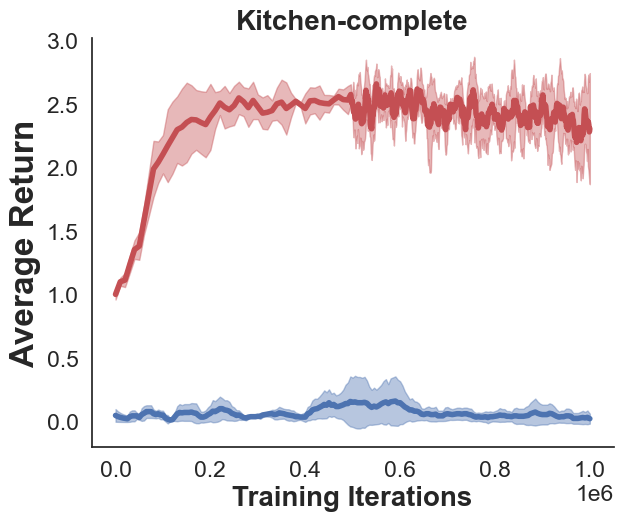} 
\label{fig:kitchen-complete-v0_3bcqcd.png} 
\end{subfigure}%
~ 
\begin{subfigure}[t]{ 0.2\textwidth} 
\centering 
\includegraphics[width=\textwidth]{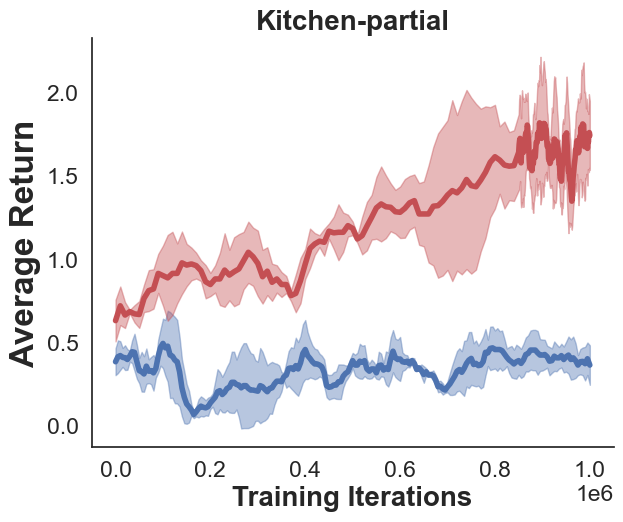} 
\label{fig:kitchen-partial-v0_3bcqcd.png} 
\end{subfigure}%
~ 
\begin{subfigure}[t]{ 0.2\textwidth} 
\centering 
\includegraphics[width=\textwidth]{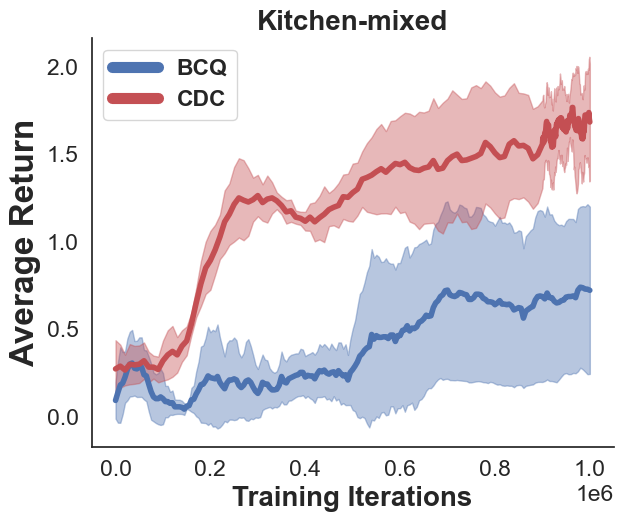} 
\label{fig:kitchen-mixed-v0_3bcqcd.png} 
\end{subfigure}%

\caption{Learning curves of CDC (red) and BCQ (blue) on all 32 D4RL environments. Curves are averaged over 3 seeds, with the shaded area showing the standard deviation across seeds.} 
\label{fig:learning_cureve_bcq_cdc} 
\end{figure} 

\paragraph{Evaluation Procedure.} We measure performance in each task using the rewards collected by the learned policy when actually deployed in the environment. To report more stable results, we follow \citesi{kumarConservativeQLearningOffline2020} and average returns achieved by each of the policies arising during the last 10k gradient steps of batch RL (done for all batch RL methods). To further improve stability, we also re-run all batch RL methods with 3 different random seeds and take another average across the resulting performance. 
As suggested by \citesi{fu2020d4rl}, we report returns for each task that have been  normalized as follows:
\begin{equation}
\label{eq:score_norm}
    \text{score} = 100 * \frac{\text{score} -\text{random score}}{\text{expert score} - \text{random score} }
\end{equation}
where \textit{random score} and \textit{expert score} are provided for each task by \citesi{fu2020d4rl} in the D4RL paper \href{https://github.com/rail-berkeley/d4rl/blob/master/d4rl/infos.py}{GitHub  repository}\footnote{\url{https://github.com/rail-berkeley/d4rl/blob/master/d4rl/infos.py}}. The same procedure is also used in previous works~\citesi{kumarConservativeQLearningOffline2020, fu2020d4rl} to report results and compare various batch RL  methods.

\paragraph{Baselines.} 
We compare \algname{} against standard baselines and state-of-the-art batch RL methods: BEAR~\citesi{kumarStabilizingOffPolicyQLearning2019}, BRAC-V and BRAC-P~\citesi{wuBehaviorRegularizedOffline2019}, BC~\citesi{wuBehaviorRegularizedOffline2019}, CQL~\citesi{kumarConservativeQLearningOffline2020}, BCQ~\citesi{fujimotoOffPolicyDeepReinforcement2019}, and SAC~\citesi{haarnoja2018soft}. We obtained the numbers for BEAR, BC, BRAC-V, and BRAC-P from published numbers by ~\citesi{kumarConservativeQLearningOffline2020}. However, numbers for BCQ and SAC are from our runs for all tasks. Also, we run published CQL codes\footnote{ \url{https://github.com/aviralkumar2907/CQL}.} with their hyperparameters to produce results for all but Adroit and FrankaKitchen where the codes are not available.  For these latter domains, we simply use the CQL results reported in the paper of \citesi{kumarConservativeQLearningOffline2020}. 
In head-to-head comparisons against each of these other batch RL methods, \algname{} generates greater overall returns (see \tabref{tab:full_d4rl}). To verify these results are statistically significant, we report the $p$-value of a (one-sided) Wilcoxon signed rank test~\citesi{Wilcoxon1945} comparing the returns of another method vs the returns of CDC across all 32 datsets (see $\mathbf{p}$\textbf{-value vs.\ CDC} row in \tabref{tab:full_d4rl}). 

To provide a better picture of our method, we also include the learning curves in \figref{fig:learning_cureve_bcq_cdc} for our algorithm vs BCQ for each environment considered in our benchmark. These plots show that, in the vast majority of environments,  \algname{} exhibits consistently better performance across different seeds/iterations.

 \begin{table*}[ht!]
 \centering 
 \begin{tabular}{c|l||c|c|c||c} 
 \hline 
 \textbf{Index} &  \textbf{Task Name}  & \textbf{ $\lambda$ = 0 \& $\eta$ = 0  } & \textbf{ $\eta$ = 0  } & \textbf{ $\lambda$ = 0  } & \textbf{ CDC  } \\ \hline   
\textbf{\textcolor{color_0}{0}} & halfcheetah-random & $\mathbf{32.8}$ & $28.78$ & $30.66$ & $27.36$\\
\textbf{\textcolor{color_0}{1}} & halfcheetah-medium & $\mathbf{49.51}$ & $48.25$ & $47.61$ & $46.05$\\
\textbf{\textcolor{color_0}{2}} & halfcheetah-medium-replay & $22.72$ & $30.47$ & $44.62$ & $\mathbf{44.74}$\\
\textbf{\textcolor{color_0}{3}} & halfcheetah-medium-expert & $7.12$ & $4.98$ & $26.99$ & $\mathbf{59.64}$\\
\textbf{\textcolor{color_0}{4}} & halfcheetah-expert & $-0.95$ & $-0.96$ & $8.21$ & $\mathbf{82.05}$\\
\textbf{\textcolor{color_1}{5}} & hopper-random & $1.58$ & $0.84$ & $5.97$ & $\mathbf{14.76}$\\
\textbf{\textcolor{color_1}{6}} & hopper-medium & $0.58$ & $1.05$ & $2.71$ & $\mathbf{60.39}$\\
\textbf{\textcolor{color_1}{7}} & hopper-medium-replay & $16.4$ & $8.8$ & $54.01$ & $\mathbf{55.89}$\\
\textbf{\textcolor{color_1}{8}} & hopper-medium-expert & $18.07$ & $3.64$ & $18.22$ & $\mathbf{86.9}$\\
\textbf{\textcolor{color_1}{9}} & hopper-expert & $1.27$ & $0.8$ & $10.89$ & $\mathbf{102.75}$\\
\textbf{\textcolor{color_2}{10}} & walker2d-random & $2.96$ & $1.55$ & $\mathbf{14.37}$ & $7.22$\\
\textbf{\textcolor{color_2}{11}} & walker2d-medium & $0.33$ & $0.85$ & $81.93$ & $\mathbf{82.13}$\\
\textbf{\textcolor{color_2}{12}} & walker2d-medium-replay & $3.81$ & $-0.14$ & $\mathbf{24.48}$ & $22.96$\\
\textbf{\textcolor{color_2}{13}} & walker2d-medium-expert & $2.65$ & $5.22$ & $10.94$ & $\mathbf{70.91}$\\
\textbf{\textcolor{color_2}{14}} & walker2d-expert & $-0.1$ & $-0.42$ & $13.03$ & $\mathbf{87.54}$\\
\textbf{\textcolor{color_3}{15}} & antmaze-umaze & $22.22$ & $89.26$ & $17.78$ & $\mathbf{91.85}$\\
\textbf{\textcolor{color_3}{16}} & antmaze-umaze-diverse & $0.0$ & $19.26$ & $0.0$ & $\mathbf{62.59}$\\
\textbf{\textcolor{color_3}{17}} & antmaze-medium-play & $0.0$ & $17.78$ & $1.11$ & $\mathbf{55.19}$\\
\textbf{\textcolor{color_3}{18}} & antmaze-medium-diverse & $0.0$ & $36.67$ & $0.0$ & $\mathbf{40.74}$\\
\textbf{\textcolor{color_3}{19}} & antmaze-large-play & $0.0$ & $\mathbf{5.56}$ & $0.0$ & $5.19$\\
\textbf{\textcolor{color_3}{20}} & antmaze-large-diverse & $0.0$ & $2.96$ & $0.0$ & $\mathbf{11.85}$\\
\textbf{\textcolor{color_4}{21}} & pen-human & $-3.43$ & $-3.07$ & $58.33$ & $\mathbf{73.19}$\\
\textbf{\textcolor{color_4}{22}} & pen-cloned & $-3.4$ & $-2.29$ & $\mathbf{49.31}$ & $49.18$\\
\textbf{\textcolor{color_5}{23}} & hammer-human & $0.26$ & $0.26$ & $0.66$ & $\mathbf{4.34}$\\
\textbf{\textcolor{color_5}{24}} & hammer-cloned & $0.26$ & $0.28$ & $1.79$ & $\mathbf{2.37}$\\
\textbf{\textcolor{color_6}{25}} & door-human & $-0.16$ & $-0.34$ & $0.01$ & $\mathbf{4.62}$\\
\textbf{\textcolor{color_6}{26}} & door-cloned & $-0.36$ & $-0.13$ & $\mathbf{0.14}$ & $0.01$\\
\textbf{\textcolor{color_7}{27}} & relocate-human & $-0.31$ & $-0.31$ & $0.0$ & $\mathbf{0.73}$\\
\textbf{\textcolor{color_7}{28}} & relocate-cloned & $\mathbf{-0.15}$ & $-0.34$ & $-0.25$ & $-0.24$\\
\textbf{\textcolor{color_8}{29}} & kitchen-complete & $0.0$ & $0.0$ & $11.76$ & $\mathbf{58.7}$\\
\textbf{\textcolor{color_8}{30}} & kitchen-partial & $0.0$ & $0.0$ & $13.52$ & $\mathbf{42.5}$\\
\textbf{\textcolor{color_8}{31}} & kitchen-mixed & $0.0$ & $0.0$ & $7.04$ & $\mathbf{42.87}$\\
 \hline \hline 
 & Total Score & $173.67$ & $299.25$ & $555.84$ & $\mathbf{1396.99}$ \\ 
 \hline
 \end{tabular}
 \caption{\textbf{Ablation study of components used in  \algname{}.} Listed is the return in each environment (normalized using \eqref{eq:score_norm} as in \cite{fu2020d4rl}) achieved by ablated variants of our algorithm. Fixing $\eta$ or $\lambda$ to zero (i.e.\ omitting our penalties) produces far worse returns than \algname{}, demonstrating the utility of both of our proposed penalties. Note that the \textbf{only} difference between \algname{} and these variants (i.e. $\lambda= 0~\&~\eta = 0$, $\eta = 0$, $\lambda=0$) in these experiments is either $\eta$ or $\lambda$ or both are set to zero in \algref{algo:cdc} and all other details are \textbf{exactly} the same.}
 \normalsize
 \label{tab:ablation}
 \end{table*}
 
\begin{table*}[h!]
 \centering 
 \begin{tabular}{l||c|c|c||c} 
 \hline 
 \textbf{Task Name}  & \textbf{ $\lambda$ = 0 \& $\eta$ = 0  } & \textbf{ $\eta$ = 0  } & \textbf{ $\lambda$ = 0  } & \textbf{ CDC  } \\ \hline   
halfcheetah-random & $\mathbf{3791.65}$ & $3293.31$ & $3526.06$ & $3117.23$\\
halfcheetah-medium & $\mathbf{5865.97}$ & $5710.67$ & $5631.1$ & $5437.01$\\
halfcheetah-medium-replay & $2540.07$ & $3503.24$ & $5259.18$ & $\mathbf{5274.51}$\\
halfcheetah-medium-expert & $604.05$ & $337.88$ & $3070.63$ & $\mathbf{7124.4}$\\
halfcheetah-expert & $-398.72$ & $-399.66$ & $739.54$ & $\mathbf{9906.71}$\\
hopper-random & $31.08$ & $7.19$ & $174.13$ & $\mathbf{459.99}$\\
hopper-medium & $-1.48$ & $13.77$ & $68.07$ & $\mathbf{1945.29}$\\
hopper-medium-replay & $513.53$ & $266.05$ & $1737.61$ & $\mathbf{1798.66}$\\
hopper-medium-expert & $567.85$ & $98.11$ & $572.74$ & $\mathbf{2808.05}$\\
hopper-expert & $20.98$ & $5.92$ & $334.25$ & $\mathbf{3323.93}$\\
walker2d-random & $137.71$ & $72.86$ & $\mathbf{661.34}$ & $333.2$\\
walker2d-medium & $16.82$ & $40.53$ & $3762.54$ & $\mathbf{3771.93}$\\
walker2d-medium-replay & $176.51$ & $-4.86$ & $\mathbf{1125.45}$ & $1055.62$\\
walker2d-medium-expert & $123.48$ & $241.43$ & $504.05$ & $\mathbf{3257.06}$\\
walker2d-expert & $-3.01$ & $-17.79$ & $599.6$ & $\mathbf{4020.49}$\\
antmaze-umaze & $0.22$ & $0.89$ & $0.18$ & $\mathbf{0.92}$\\
antmaze-umaze-diverse & $0.0$ & $0.19$ & $0.0$ & $\mathbf{0.63}$\\
antmaze-medium-play & $0.0$ & $0.18$ & $0.01$ & $\mathbf{0.55}$\\
antmaze-medium-diverse & $0.0$ & $0.37$ & $0.0$ & $\mathbf{0.41}$\\
antmaze-large-play & $0.0$ & $\mathbf{0.06}$ & $0.0$ & $0.05$\\
antmaze-large-diverse & $0.0$ & $0.03$ & $0.0$ & $\mathbf{0.12}$\\
pen-human & $-5.87$ & $4.84$ & $1834.73$ & $\mathbf{2277.75}$\\
pen-cloned & $-5.04$ & $28.09$ & $\mathbf{1565.86}$ & $1562.21$\\
hammer-human & $-241.31$ & $-241.06$ & $-188.27$ & $\mathbf{292.21}$\\
hammer-cloned & $-241.4$ & $-238.34$ & $-41.52$ & $\mathbf{34.45}$\\
door-human & $-61.11$ & $-66.61$ & $-56.14$ & $\mathbf{79.05}$\\
door-cloned & $-67.09$ & $-60.23$ & $\mathbf{-52.37}$ & $-56.15$\\
relocate-human & $-19.51$ & $-19.72$ & $-6.37$ & $\mathbf{24.36}$\\
relocate-cloned & $\mathbf{-12.77}$ & $-20.95$ & $-17.05$ & $-16.63$\\
kitchen-complete & $0.0$ & $0.0$ & $0.47$ & $\mathbf{2.35}$\\
kitchen-partial & $0.0$ & $0.0$ & $0.54$ & $\mathbf{1.7}$\\
kitchen-mixed & $0.0$ & $0.0$ & $0.28$ & $\mathbf{1.71}$\\
 \hline \hline 
Total Score & $13332.61$ & $12556.37$ & $30806.65$ & $\mathbf{57839.77}$ \\ 
 \hline
 \end{tabular}
 \caption{\textbf{Ablation study of components used in  \algname{}.} Same as \tabref{tab:ablation} but 
  the returns here are not normalized, and we instead report raw returns achieved in each task. }
 \normalsize
 \label{tab:raw_ablation}
 \end{table*}

\begin{figure}[!t]
\centering\captionsetup[subfigure]{justification=centering}
\begin{subfigure}[t]{0.25\textwidth}
\centering
\includegraphics[width=\textwidth]{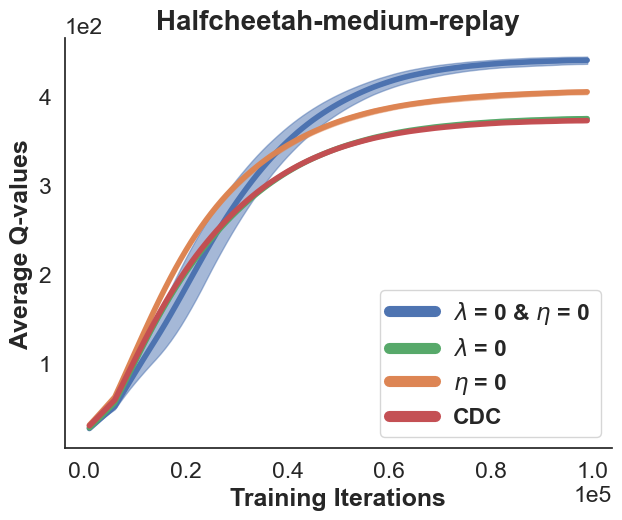}
\label{fig:halfcheetah-medium-replay_avg_qs}
\end{subfigure}%
~
\begin{subfigure}[t]{0.25\textwidth}
\centering
\includegraphics[width=\textwidth]{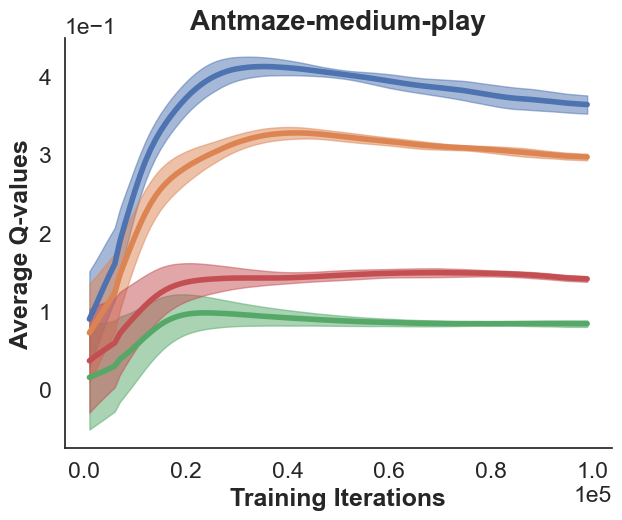}
\label{fig:antmaze-medium-play_avg_qs}
\end{subfigure}%
~
\begin{subfigure}[t]{0.25\textwidth}
\centering
\includegraphics[width=\textwidth]{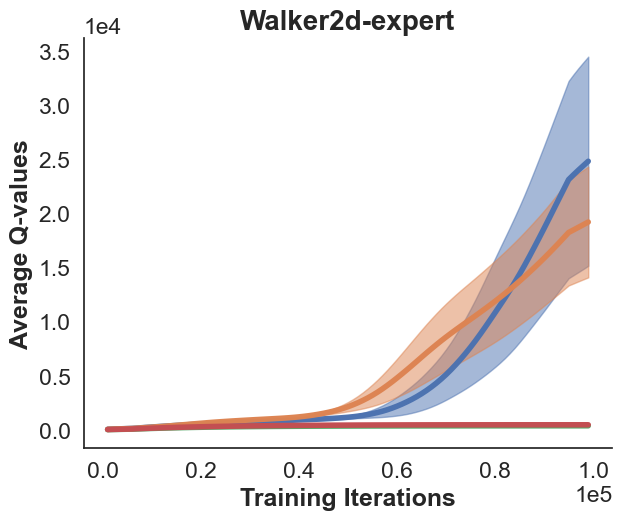}
\label{fig:walker2d-expert_avg_qs}
\end{subfigure}%

\begin{subfigure}[t]{0.25\textwidth}
\centering
\includegraphics[width=\textwidth]{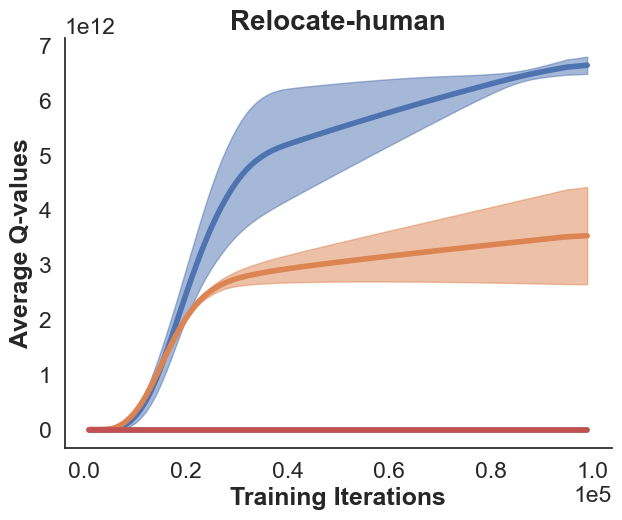}
\label{fig:relocate-human_avg_qs}
\end{subfigure}%
~
\begin{subfigure}[t]{0.25\textwidth}
\centering
\includegraphics[width=\textwidth]{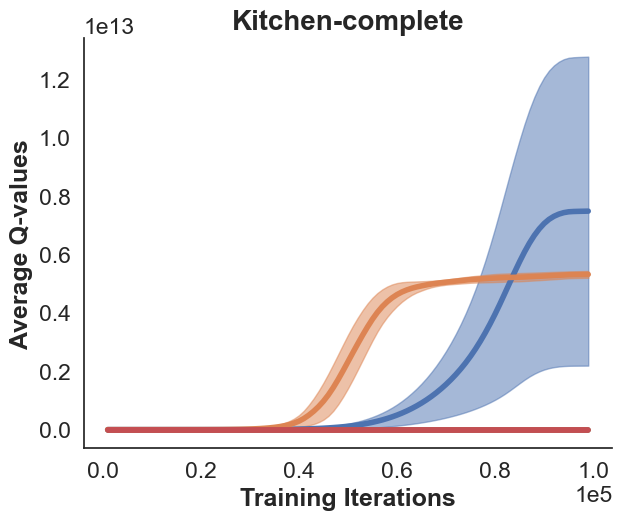}
\label{fig:kitchen-complete_avg_qs}
\end{subfigure}%
~
\begin{subfigure}[t]{0.25\textwidth}
\centering
\includegraphics[width=\textwidth]{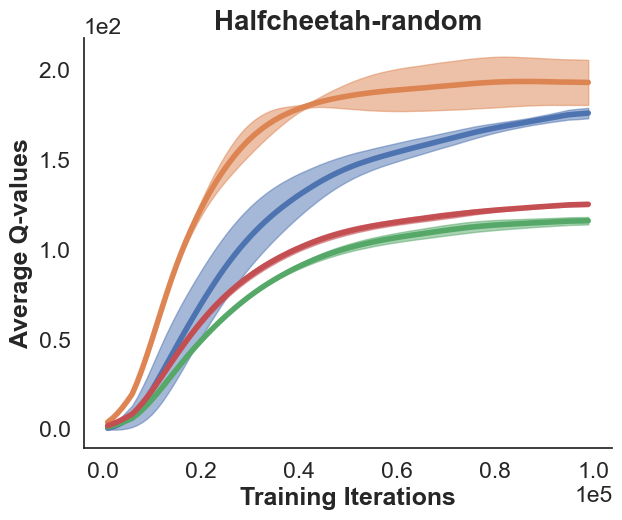}
\label{fig:halfcheetah-random-v0_avg_qs}
\end{subfigure}%
\vspace*{-1.5em}
\caption{\textbf{Effect of our penalties on Q-values.}
These figures show evaluation of averaged Q values across 4 Q during training time for 6 different tasks.  This result shows that \algname{}'s Q-estimate is well controlled especially compared with $\eta=0$.}
\label{fig:avg_qs}
\end{figure}

\paragraph{Ablation studies.}
We conduct a series of ablation studies to comprehensively analyze  the different components of \algname{}. We use all 32 D4RL datasets for this purpose. \tabref{tab:ablation} and \tabref{tab:raw_ablation} show that 
both penalties introduced in our paper are critical for the strong performance of \algname{}, with the extra-overestimation penalty $\Delta$ being of greater importance than the exploration-penalty $\log \pi$. Moreover, \figref{fig:avg_qs} shows how estimated Q values evolve over training for each of the above ablation variants. Here it is again evident that both penalties may be required to successfully prevent  extra-overestimation and subsequent explosion of Q-estimates, with $\Delta$ being the more effective of the two for mitigating extra-overestimation.

\subsection{Comparing CDC with RBVE (Gulcehre et al.~\cite{gulcehre2021regularized}) }\label{secapx:rbve-vs-cdc}

Concurrent to our work, \citetsi{gulcehre2021regularized} propose a value regularization term for batch RL that is similar to the Q-value regularizer used by \algname{}.  Unlike our work, \citesi{gulcehre2021regularized} only considers discrete actions under a DQN~\citesi{mnih2015human} framework rather than the actor-critic RL framework employed in \algname{}. \citesi{gulcehre2021regularized} also does not consider explicit policy regularization, which forms a critical component of \algname{} to supplement its value regularization.  
That said, \cite{gulcehre2021regularized} do also acknowledge the importance of ensuring the learned policy does not stray too far from the behavior policy, but their proposal to ensure this involves restricting the learner to apply only a single step of policy-iteration to the estimated value function. However in continuous action spaces with a policy-network, even a single policy-iteration step can lead to large deviations from the behavior policy without explicit policy regularization as imposed by CDC. Finally, we note that the RBVE methodology of \citesi{gulcehre2021regularized} requires a dataset that contains observations $(s, a, r, s', a')$, i.e.\ more complete subtrajectories of episodes, whereas CDC merely requires a dataset that contains observations of the form $(s, a, r, s')$. The former setting is less widely applicable, but is somewhat easier due to the availability of the subsequent action $a'$ for temporal-difference learning.

In this section, we apply the RBVE method of~\citesi{gulcehre2021regularized} on the D4RL benchmark~\citesi{fu2020d4rl}, after first minorly adapting it to our setting. Key  differences in our setting are: we have continuous actions, and  $a'$ is not contained in the dataset. In our adaptation of RBVE, we approximate the $\max_a$ required by \citesi{gulcehre2021regularized} (but which is difficult for continuous actions) by sampling many actions and taking the empirical maximum. Our adaptation accounts for the fact that $a'$ is not present in the dataset by first estimating the behavior policy $\pi_b$ via behavior-cloning (i.e.\ via maximum likelihood training of our same policy network) and then drawing $a' \sim \pi_b( . | s')$ for use in RBVE. Furthermore, we considered two different variants of RBVE in our experiments. In the first variant (called RBVE-A), the soft filtering weights of \citesi{gulcehre2021regularized}, $\omega(s,a)$, are implemented according to Eq 6 in their paper. Although closely following the recommendations of \citesi{gulcehre2021regularized}, RBVE-A did not perform well in our D4RL environments, and thus we considered a second variant (called RBVE-C), where we treat $\omega$ as a hyperparameter and we use a fixed value per environment. 
\tabref{tab:rbve_d4rl} illustrates that CDC outperforms both variants of RBVE. 

\begin{table*}[tp]
 \centering 
 \begin{tabular}{c|l||c|c|c|c} 
 \hline 
 \textbf{Index} &  \textbf{Task Name}  & \textbf{ RBVE-C } & \textbf{ RBVE-A } & \textbf{ CDC } \\ \hline   
\textbf{\textcolor{color_0}{0}} & \small{halfcheetah-random}  & 18.89 & -0.01 & $\mathbf{27.36}$\\
\textbf{\textcolor{color_0}{1}} & \small{halfcheetah-medium}  & 43.98 & 24.27 & $\mathbf{46.05}$\\
\textbf{\textcolor{color_0}{2}} & \small{halfcheetah-medium-replay}  & 37.24 & 7.55 & $\mathbf{44.74}$\\
\textbf{\textcolor{color_0}{3}} & \small{halfcheetah-medium-expert}  & 32.11 & 8.12 & $\mathbf{59.64}$\\
\textbf{\textcolor{color_0}{4}} & \small{halfcheetah-expert}  & 36.57 & 3.14 & $\mathbf{82.05}$\\
\textbf{\textcolor{color_1}{5}} & \small{hopper-random}  & 11.46 & 4.69 & $\mathbf{14.76}$\\
\textbf{\textcolor{color_1}{6}} & \small{hopper-medium}  & 17.16 & 1.23 & $\mathbf{60.39}$\\
\textbf{\textcolor{color_1}{7}} & \small{hopper-medium-replay}  & 28.15 & 3.31 & $\mathbf{55.89}$\\
\textbf{\textcolor{color_1}{8}} & \small{hopper-medium-expert}  & $\mathbf{88.72}$ & 3.62 & 86.9\\
\textbf{\textcolor{color_1}{9}} & \small{hopper-expert}  & 94.32 & 1.2 & $\mathbf{102.75}$\\
\textbf{\textcolor{color_2}{10}} & \small{walker2d-random}  & 0.36 & 2.4 & $\mathbf{7.22}$\\
\textbf{\textcolor{color_2}{11}} & \small{walker2d-medium}  & 80.19 & 2.81 & $\mathbf{82.13}$\\
\textbf{\textcolor{color_2}{12}} & \small{walker2d-medium-replay}  & 6.7 & 2.41 & $\mathbf{22.96}$\\
\textbf{\textcolor{color_2}{13}} & \small{walker2d-medium-expert}  & $\mathbf{77.79}$ & 1.83 & 70.91\\
\textbf{\textcolor{color_2}{14}} & \small{walker2d-expert}  & 60.3 & 0.69 & $\mathbf{87.54}$\\
\textbf{\textcolor{color_3}{15}} & \small{antmaze-umaze}  & 0.0 & 0.0 & $\mathbf{91.85}$\\
\textbf{\textcolor{color_3}{16}} & \small{antmaze-umaze-diverse}  & 2.96 & 0.0 & $\mathbf{62.59}$\\
\textbf{\textcolor{color_3}{17}} & \small{antmaze-medium-play}  & 0.0 & 0.0 & $\mathbf{55.19}$\\
\textbf{\textcolor{color_3}{18}} & \small{antmaze-medium-diverse}  & 0.0 & 0.0 & $\mathbf{40.74}$\\
\textbf{\textcolor{color_3}{19}} & \small{antmaze-large-play}  & 0.0 & 0.0 & $\mathbf{5.19}$\\
\textbf{\textcolor{color_3}{20}} & \small{antmaze-large-diverse}  & 0.0 & 0.0 & $\mathbf{11.85}$\\
\textbf{\textcolor{color_4}{21}} & \small{pen-human}  & 24.04 & 34.93 & $\mathbf{73.19}$\\
\textbf{\textcolor{color_4}{22}} & \small{pen-cloned}  & 42.86 & -0.39 & $\mathbf{49.18}$\\
\textbf{\textcolor{color_5}{23}} & \small{hammer-human}  & 0.59 & 0.0 & $\mathbf{4.34}$\\
\textbf{\textcolor{color_5}{24}} & \small{hammer-cloned}  & 0.34 & 0.17 & $\mathbf{2.37}$\\
\textbf{\textcolor{color_6}{25}} & \small{door-human}  & $\mathbf{9.15}$ & -0.0 & 4.62\\
\textbf{\textcolor{color_6}{26}} & \small{door-cloned}  & $\mathbf{0.04}$ & 0.03 & 0.01\\
\textbf{\textcolor{color_7}{27}} & \small{relocate-human}  & 0.29 & 0.01 & $\mathbf{0.73}$\\
\textbf{\textcolor{color_7}{28}} & \small{relocate-cloned}  & -0.23 & $\mathbf{0.01}$ & -0.24\\
\textbf{\textcolor{color_8}{29}} & \small{kitchen-complete}  & 18.61 & 0.83 & $\mathbf{58.7}$\\
\textbf{\textcolor{color_8}{30}} & \small{kitchen-partial}  & 8.7 & 0.83 & $\mathbf{42.5}$\\
\textbf{\textcolor{color_8}{31}} & \small{kitchen-mixed}  & 5.93 & 1.11 & $\mathbf{42.87}$\\
 \hline \hline 
 & Total Score & $747.21$ & $104.79$ & $\mathbf{1396.99}$ \\ 
 \hline \hline 
 & $\mathbf{p}$\textbf{-value vs.\ CDC} & 6.6e-06 & 5.3e-07 &  -  \\ 
 \hline
 \end{tabular}
 \caption{{ \textbf{Return achieved in deployment of policies learned via \algname{} and RBVE~\cite{gulcehre2021regularized}.} RBVE-C and RBVE-A are two variants of \cite{gulcehre2021regularized} detailed in  \secref{secapx:rbve-vs-cdc}, and we use the exact same setup for \algname{} as before. The return in each environment here is normalized using \eqref{eq:score_norm} as originally advocated by \cite{fu2020d4rl}.}
 For each method: we perform a head-to-head comparison against \algname{} across the D4RL tasks, reporting the $p$-value of a (one-sided) Wilcoxon signed rank test~\cite{Wilcoxon1945} that compares this method's return against that of \algname{} (over the 32 tasks).  }
  \label{tab:rbve_d4rl}
 \end{table*}

\section{Details of our Methods}\label{sec:app:hyperparams}
\paragraph{Implementation Details.}
\tabref{tab:table_exp_hypers} and \tabref{tab:table_compute_hypers} show hyper-parameters, computing infrastructure, and libraries used for the experiments in this paper for all 32 continuous-control tasks. Like most other batch-RL baselines in our comparisons and following Sec 8 in \citesi{agarwal2020optimistic}, we did a minimal random search to tune our hyperparams $\eta, \lambda$. Note that \algname{} was simply run on every task using the same network architecture and the original rewards/actions provided in the task, without any task-specific reward-normalization/action-smoothing required by some of the other batch RL methods.

\paragraph{Using \algname{} Policy During Deployment.} \algref{alg:ours} in the main text only describes the training process used in \algname{}, \algref{alg:ours_test} here details how a batch RL is deployed the resulting learned policy/values to select actions in the actual environment.
After the batch RL training is complete, \algref{alg:ours_test} is used to select actions at  evaluation (test) time, as also done by \citesi{fujimotoOffPolicyDeepReinforcement2019,wuBehaviorRegularizedOffline2019, kumarConservativeQLearningOffline2020,ghasemipour2021emaq, kumarStabilizingOffPolicyQLearning2019}.  
All \algname{} returns mentioned throughout the paper (and other baseline methods returns i.e. BCQ, CQL, BEAR, BRAC-V/P, etc.) were produced by selecting actions in this manner, which facilitates fair comparison against the existing literature.

\resizebox{\columnwidth}{!}{
\begin{minipage}[t]{0.5\linewidth}
\begin{algorithm}[H]
\caption{ FQE + $\Delta$}
\begin{algorithmic}[1]
    \STATE \textbf{Input}: policy $\pi$ to evaluate 
    \STATE Initialize Q networks: $\{Q_{\theta_j} \}_{j=1}^{M}$
    \STATE Initialize Target Q-networks:  $\{Q_{\theta'_j}: \theta'_j \leftarrow \theta_j\}_{j=1}^{M}$   \FOR{$t$ in \{1, \dots, T\}}
        \STATE Sample mini-batch $\mathcal{B} = \{(s, a, r, s')\} \sim \mathcal{D}$ \\
        \STATE For each $s, s' \in \mathcal{B}$:
        sample $N$ actions $\{\hat{a}_{k} \}_{k=1}^N \sim \pi(\cdot|s)$ and   $\{{a}'_{k}\}_{k=1}^N \sim \pi(\cdot|s')$ \\[0.5em]
        \STATE \textbf{$Q_{\theta}$- value update:} \\ 
        \vspace*{-.2in}
            \begin{align*}
                &y(s', r) := r + \frac{\gamma}{N} \sum_{a'_k}^N \Big[\overline{Q}_{\theta'}(s', a'_k) \Big] \ \
                 \text{where $\overline{Q}$ given by Eq~\ref{eq:qcvx}} \\
                & {\textcolor{color_alg}{ \Delta_j(s,a) :=  \Big(  \Big[  \max_{\hat{a}_{k}}  Q_{\theta_j}(s, \hat{a}_{k}) -  Q_{\theta_j}(s, a) \Big]_+ \Big)^2 } }
                \\
                & \theta_j \leftarrow \hspace*{0mm} \argmin_{\theta_j} \hspace*{-2mm} \sum_{(s, a, s') \in \mathcal{B}} \Big[ \Big( Q_{\theta_j}(s, a) - y(s',r) \Big )^2  \\[-0.78em]
                & \hspace*{30mm}
                 {\textcolor{color_alg}{ \ + \ \eta \cdot \Delta_j(s,a) }} \Big]
                \ \text{ for } j = 1,...,M
                \\[-2em]
            \end{align*}
        \STATE \textbf{Update Target Networks: }
              \\[0.5em]
              $~~~~~~~\theta'_j \leftarrow \tau \theta_j + (1 - \tau)\theta'_j~~\forall j \in M $
              \\[-0.5em]
    \ENDFOR
\end{algorithmic}
\label{algo:fqe_cdc}
\end{algorithm}
\end{minipage}
\begin{minipage}[t]{0.5\linewidth}
\begin{algorithm}[H]
\caption{Bacth RL Action Selection at Evaluation Time (used for \algname{} as well as other baseline methods)}
\begin{algorithmic}[1]
    \STATE \textbf{Input}: state $s \in S$, trained policy network $\pi_\phi$ and Q networks: $\{Q_{\theta_j} \}_{j=1}^{M}$.
    \STATE Sample $N$ actions $\{a_{k} \}_{k=1}^N \sim \pi_\phi(\cdot|s)$
     \STATE Identify optimal action:
      \begin{align*}
         a \leftarrow  \arg\max_{{a}_{k}}  &\Big[\overline{Q}_{\theta}(s, {a}_{k}) \Big]&
     \end{align*}
    Here $\overline{Q}$ given by Eq~\ref{eq:qcvx} (Note this similar to BCQ, BRAC-V/P, BEAR, EMaQ).   
    \STATE \textbf{Return} $a$
\end{algorithmic}
\label{alg:ours_test}
\end{algorithm}
\end{minipage}
}

\subsection{Fitted Q Evaluation Details}
For off-policy evaluation, \algref{algo:fqe_cdc} describes the steps of fitted Q-evaluation (FQE)~\citesi{Hoangfqe19}, when additionally leveraging our  extra-overestimation penalty $\Delta$. The goal of FQE is to
estimate the values for a given policy, i.e. $\hat{Q}^\pi$, with offline data collected by an unknown behavior policy. After learning an estimate $\hat{Q}^\pi$, the resulting Q-values will be used to score a policy $\pi$ via the expectation of $\hat{Q}^\pi$ over the initial state distribution and actions proposed by this policy, i.e.\ the estimated expected return, which is given by $\hat{v}(\pi) = \mathbb{E}_{s \sim \mathcal{D}} \mathbb{E}_{a \sim \pi(\cdot | s)} [\hat{Q}^\pi(s,a) ]$ \citesi{Hoangfqe19}. 
Applying Q-learning to limited data, FQE is also prone to suffer from wild extrapolation, which we attempt to mitigate by introducing our $\Delta$ penalty (highlighted blue terms in \algref{algo:fqe_cdc} are our modifications to FQE).

In \figref{fig:fqe_pearson} of \secref{sec:exppolicyevaluation} in the main text, we compare the performance of $\Delta$-penalization of FQE (with $\eta=1$ throughout in \algref{algo:fqe_cdc}) against the standard unregularized FQE ($\eta=0$  in \algref{algo:fqe_cdc}). Here we use both methods to score $20$ different policies, learned under CDC with different random hyperparameter settings.  
When scoring each CDC-policy, $a \sim \pi$ in the definition of $\hat{v}$ is obtained using \algref{alg:ours_test} for each $s$, as the operations of \algref{alg:ours_test} entail the actual policy considered for deployment.  

Here we assess the quality of FQE policy evaluation via the Pearson correlation between estimated returns, i.e., $\hat{v}(\pi)$, and the actual return (over our $20$ policies under consideration). 
The higher correlations observed for FQE + $\Delta$ (0.37 on average across our 32 tasks) over FQE (0.01 on average) in the majority of tasks demonstrates how the inclusion of our $\Delta$ penalty can lead to more reliable off policy evaluation estimates. Our strategies for mitigating overestimation/extrapolation are thus not only useful for  batch RL but also related tasks like off-policy evaluation.

\begin{table*}
    \centering
    \begin{tabular}{|l|c|}
        \hline
        \multicolumn{2}{|c|}{\textbf{Computing Infrastructure}} \\ \hline
        Machine Type & AWS EC2 - p2.16xlarge\\
        CUDA Version & $10.2$ \\
        NVIDIA-Driver & $440.33.01$ \\
        GPU Family & Tesla K80 \\
        CPU Family & Intel Xeon 2.30GHz \\ \hline
       \multicolumn{2}{|c|}{\textbf{Library Version}} \\ \hline
        Pytorch & 1.6.0 \\
        Gym & 0.17.2 \\
        Python & 3.7.7 \\
        Numpy &  1.19.1 \\
        \hline
   \end{tabular}
    \caption{Computing infrastructure and software libraries used in this paper.}
    \label{tab:table_compute_hypers}
\end{table*}

\begin{table*}[t!]
    \centering
    \begin{tabular}{|l|c|}
        \hline
        Hyper-parameters & value \\
        \hline
        Random Seeds & $\{0,1,2\}$ \\
        Overestimation bias coef ($\nu$) & 0.75 \\
        Batch Size & 256 \\
        Number of Updates & 1e+6 \\
        Number of $Q$ Functions & 4 \\
        Number of hidden layers (Q) & 4 layers \\
        Number of hidden layers ($\pi$) & 4 layers \\
        Number of hidden units per layer &  256 \\
        Number of sampled actions ($N$) &  15 \\
        Nonlinearity & \textit{ReLU}\\
        Discount factor ($\gamma$) & 0.99 \\
        Target network ($\theta'$) update rate ($\tau$) &  0.005 \\
        Actor learning rate & 3e-4\\
        Critic learning rate & 7e-4\\
        Optimizer & Adam\\
        Policy constraint coef ($\lambda$) & $\{0.1, 0.5, 1, 2\}$\\
        Extra-overestimation coef ($\eta$) & \small{$\{.1, .2, .5,.6, .8, 1, 5, 20, 25, 50, 100, 200\}$} \\
        Number of episodes to evaluate & 10 \\ \hline
   \end{tabular}
    \caption{Hyper-parameters used for \algname{} for all 32 continuous-control tasks in the D4RL benchmark. All results reported in our paper are averages over repeated runs initialized with \emph{each} of the random seeds listed above and run for the listed number of episodes.
    }
    \label{tab:table_exp_hypers}
\end{table*}
\section{D4RL Benchmark}\label{sec:app:d4rl_detail}
D4RL is a large-scale benchmark for evaluating batch RL algorithms~\cite{fu2020d4rl}. It contains  many diverse tasks with different levels of complexity in which miscellaneous behavior policies (ranging from random actions to expert demonstrations) have been used to collect data. For each task, batch RL agents are trained on a large offline dataset $\mathcal{D}$ (without environment interaction), and these agents are scored based on how much return they produce when subsequently deployed into the live environment.
Since the benchmark contains multiple tasks from a single environment (with different $\pi_b$), we can observe how well batch RL methods are able to learn from behavior policies of different quality.

We consider four different domains from the D4RL benchmark~\citesi{fu2020d4rl} from which  $\mathbf{32}$ different datasets (i.e.\ tasks) are available. 
Each dataset here corresponds to a single batch RL task, where we treat the provided data as $\mathcal{D}$, learn a policy $\pi$ using only $\mathcal{D}$, and finally evaluate this policy when it is deployed in the actual environment.  
In many cases, we have two different datasets taken from the same environment, but collected by behavior policies of varying quality. For example, from the MuJoCo HalfCheetah environment, we have one dataset (HalfCheetah-random) generated under a behavior policy that randomly selects actions and another dataset (HalfCheetah-expert) generated under an expert behavior policy that generates strong returns. Note that our batch RL agents do not have information about the quality of $\pi_b$ since this is often unknown in practice. 

The \textit{Gym-MuJoCo} domain consists of four environments (Hopper, HalfCheetah, Walker2d) from which we have 15 datasets  built by mixing different behavior policies. Here \citesi{fu2020d4rl} wanted to examine the effectiveness of a given batch RL method for learning under heterogeneous $\pi_b$.
The \textit{FrankaKitchen} domain is based on a 9-degree-of-freedom (DoF) Franka robot in a kitchen environment containing various household items. There are $3$ datasets from this environment designed to evaluate the generalization of a given algorithm to unseen states~\citesi{fu2020d4rl}. 
The \textit{Adroit} domain is based on a 24-DoF simulated robot hand with goals such as: hammering a nail, opening a door, twirling a pen, or picking up and moving a ball. 
\citesi{fu2020d4rl} provide $8$ datasets from this domain, hoping to study batch RL in settings with small amounts of expert data (human demonstrations) in a high-dimensional robotic manipulation task. 
Finally, \textit{AntMaze} is a navigation domain based on an 8-DoF Ant quadruped robot,  from which the benchmark contains $6$ datasets. Here \citesi{fu2020d4rl} aim to test how well batch RL agents are able to stitch together pieces of existing observed trajectories to solve a given task (rather than requiring generalization beyond  $\mathcal{D}$). 
\tabref{tab:dataset_info} shows more details about the datasets in our benchmark.

 \begin{table*}[h!] 
 \centering
 \resizebox{0.8\columnwidth}{!}{
 \begin{tabular}{c|l|r|r|r}
 \hline
 \textbf{Domain}  &  \textbf{Task Name}  & \textbf{$\#$Samples } & \textbf{Obs Dims} & \textbf{Action Dims}\\ \hline
\multirow{6}{*}{\textbf{AntMaze}}
& antmaze-umaze-v0 & $998573$ & $29$ & $8$ \\
& antmaze-umaze-diverse-v0 & $998882$ & $29$ & $8$ \\
& antmaze-medium-play-v0 & $999092$ & $29$ & $8$ \\
& antmaze-medium-diverse-v0 & $999035$ & $29$ & $8$ \\
& antmaze-large-play-v0 & $999059$ & $29$ & $8$ \\
& antmaze-large-diverse-v0 & $999048$ & $29$ & $8$ \\
\hline
\multirow{8}{*}{\textbf{Adroit}}
& pen-human-v0 & $4950$ & $45$ & $24$ \\
& hammer-human-v0 & $11264$ & $46$ & $26$ \\
& door-human-v0 & $6703$ & $39$ & $28$ \\
& relocate-human-v0 & $9906$ & $39$ & $30$ \\
& pen-cloned-v0 & $495071$ & $45$ & $24$ \\
& hammer-cloned-v0 & $995511$ & $46$ & $26$ \\
& door-cloned-v0 & $995643$ & $39$ & $28$ \\
& relocate-cloned-v0 & $995739$ & $39$ & $30$ \\
\hline
\multirow{3}{*}{\textbf{FrankaKitchen}}
& kitchen-complete-v0 & $3679$ & $60$ & $9$ \\
& kitchen-partial-v0 & $136937$ & $60$ & $9$ \\
& kitchen-mixed-v0 & $136937$ & $60$ & $9$ \\
\hline
\multirow{15}{*}{\textbf{Gym-MuJoCo}}
& halfcheetah-random-v0 & $998999$ & $17$ & $6$ \\
& hopper-random-v0 & $999999$ & $11$ & $3$ \\
& walker2d-random-v0 & $999999$ & $17$ & $6$ \\
& halfcheetah-medium-v0 & $998999$ & $17$ & $6$ \\
& walker2d-medium-v0 & $999874$ & $17$ & $6$ \\
& hopper-medium-v0 & $999981$ & $11$ & $3$ \\
& halfcheetah-medium-expert-v0 & $1997998$ & $17$ & $6$ \\
& walker2d-medium-expert-v0 & $1999179$ & $17$ & $6$ \\
& hopper-medium-expert-v0 & $1199953$ & $11$ & $3$ \\
& halfcheetah-medium-replay-v0 & $100899$ & $17$ & $6$ \\
& walker2d-medium-replay-v0 & $100929$ & $17$ & $6$ \\
& hopper-medium-replay-v0 & $200918$ & $11$ & $3$ \\
& halfcheetah-expert-v0 & $998999$ & $17$ & $6$ \\
& hopper-expert-v0 & $999034$ & $11$ & $3$ \\
& walker2d-expert-v0 & $999304$ & $17$ & $6$ \\
  \hline
 \end{tabular}}
 \caption{ \textbf{Overview of D4RL tasks}. 
 Summary of 32 datasets (i.e.\ tasks, environments) considered in this work, listing the: domain each dataset stems from, name of each task, number of samples (i.e.\ transitions) in each dataset, and the dimensionality of the state (\textbf{Obs Dims}) and action space (\textbf{Action Dims}).
To get the numbers listed here, a few samples were omitted from the original datasets using the timeout flag suggested by \cite{fu2020d4rl} (\href{https://github.com/rail-berkeley/d4rl/blob/master/d4rl/__init__.py}{Click here for details}).}  
 \normalsize
 \label{tab:dataset_info}
 \end{table*}

\clearpage
\section{Proofs and Additional Theory}
\label{sec:proofs}
This section contains proofs and details of the constants/assumptions of theories mentioned in the main text. 
\subsection{Proof of Theorem \ref{thm:contract}. }\label{app:proof_contract}
\textbf{Theorem \ref{thm:contract}.}
For $\overline{Q}_{\theta}$ in (\ref{eq:qcvx}), let   $\mathcal{T}_{\text{CDC}}: \overline{Q}_{\theta_t} \rightarrow \overline{Q}_{\theta_{t+1}}$ denote the operator corresponding to the $\overline{Q}_{\theta}$-updates resulting from the $t^{\text{th}}$ iteration of Steps \ref{step:value}-\ref{step:policy} of \algref{alg:ours}. $\mathcal{T}_{\text{CDC}}$ is a $L_\infty$ contraction under standard conditions that suffice for the ordinary Bellman operator to be contractive \citesi{bertsekas2004stochastic, busoniu2010reinforcement, szepesvari2001efficient, antos2007fitted}. That is, for any $\overline{Q}_1, \overline{Q}_2$:
\begin{align*}
 \sup_{s,a} \left| \mathcal{T}_{\text{CDC}}(\overline{Q}_1(s,a) ) - \mathcal{T}_{\text{CDC}}(\overline{Q}_2(s,a)) \right|
\le   \
\gamma \cdot \sup_{s,a} \left| \overline{Q}_1(s,a) - \overline{Q}_2(s,a) \right|  \end{align*}

In this theorem, we also assume that: $\pi_\phi$ is sufficiently flexible to produce $\arg\max_{\hat{a}} \overline{Q}(s, \hat{a})$ for all $s
\in \mathcal{D}$ and the optimization subproblems in Steps \ref{step:value}-\ref{step:policy} of \algref{alg:ours} are solved exactly (ignoring all issues related to  function approximation). More formally, we adopt assumptions A1-A9 of \citetsi{antos2007fitted}, although the result from this theorem can also be shown to hold under alternative conditions that suffice for the ordinary Bellman operator to be contractive (see Section 2 of \citetsi{antos2007fitted}). These assumptions involve regularity conditions on the underlying MDP and the behavior policy, as well as expressiveness restrictions on the hypothesis class of our neural networks.

\begin{proof}
We first consider a simple unpenalized case where $\eta = 0$, $\lambda = 0$, and $M = 1$, i.e.\  the $Q$-ensemble consists of a single network.
By the definition in (\ref{eq:qcvx}) with $M=1$: $\overline{Q}_\theta = Q_{\theta_1}$, so Step \ref{step:value} implements the standard \emph{Bellman-optimality} operator update, when we assume $\pi_\phi$ produces $a = \arg\max_{\hat{a}} \overline{Q}_\theta(s, \hat{a})$.
This operator is a contraction under standard conditions \citesi{antos2007fitted,lagoudakis2003least, busoniu2010reinforcement}.
Without this assumption on $\pi_\phi$, Step \ref{step:value} instead relies on the EMaQ operator, which Theorem 3.1 of \citetsi{ghasemipour2021emaq} shows is also a contraction for  the special case of tabular MDPs.

Next consider $M > 0$ (still with $\eta, \lambda = 0$). Now the target-value $y(s')$ for each single $Q$-network $Q_{\theta_j}$ is determined by $\overline{Q}_\theta$ rather than $Q_{\theta_j}$ alone, i.e.\ $y(s')$ is given by a convex combination of target networks $\displaystyle \{ Q_{\theta_j} \}_{j=1}^M$. By Jensen's inequality and basic properties of convexity, the updates to each
$Q_{\theta_j}$ remain a contraction. Therefore the overall update to the convex combination of these networks  $\overline{Q}_\theta$ is likewise a contraction.

Next we additionally consider $\eta > 0$. Note that reducing $\Delta_j$ is a non-expansive operation on each $Q_{\theta_j}$, since $\Delta_j(s,a)$ is reduced by shrinking large $\max_{\hat{a}} Q_{\theta_j}(s, \hat{a})$ toward $Q(s,a)$ for the $a$ observed in $\mathcal{D}$ (without modifying $Q_{\theta_j}(s, a')$ for other $a'$). Following the previous arguments, the addition of our $\Delta$ penalty preserves the contractive nature of the $\overline{Q}_\theta$ update.

Finally also consider $\lambda > 0$. In this case, $\pi_\phi$ does not simply concentrate on actions that maximize $\overline{Q}_\theta$, so Step \ref{step:value} no longer implements the  Bellman-optimality operator even with $M = 1, \eta = 0$. However with the likelihood penalty, Step \ref{step:policy} is simply a regularized \emph{policy-improvement} update:
With $\eta=0, M=1$, Step \ref{step:value} becomes a \emph{policy-evaluation} calculation where the policy being evaluated is $ \tilde{\pi}(a|s) = \arg\max_{\{{a}'_{k}\}_{k=1}^N \sim \pi_{\phi}(\cdot|s')} [ \overline{Q}_\theta ]$.
Since the \emph{Bellman-evaluation} operator is also a contraction under standard conditions \citesi{antos2007fitted, lagoudakis2003least, busoniu2010reinforcement}, our overall argument remains otherwise intact.
\end{proof}

\subsection{Proof of Theorem \ref{thm:reliable}. }\label{app:proof_reliable}
Theorem \ref{thm:reliable} (restated below) assures us of the reliability of the policy $\pi_\phi$ produced by CDC, guaranteeing that with high probability $\pi_\phi$ will not have much worse outcomes than the behavior policy $\pi_b$ (where the probability here depends on the size of the dataset $\mathcal{D}$). 
In batch settings, expecting to learn the optimal policy is futile from limited data. Even ensuring \emph{any} improvement at all over an arbitrary $\pi_b$ is ambitious when we cannot ever test any policies in the environment, and reliability of the learned $\pi_\phi$ is thus a major concern. 

\textbf{Theorem \ref{thm:reliable}.}
Let $\pi_\phi \in \Pi$ be the policy learned by \algname{}, $\gamma$ denote discount factor, and $n$ denote the sample size of dataset $\mathcal{D}$ generated from $\pi_b$. Also let $J(\pi)$ represent the true expected return  produced by deploying policy $\pi$ in the environment. 
Under assumptions (A1)-(A4), there exist constants $r^*, C_\lambda, V$ such that with high probability $\ge 1 - \delta$:  
\begin{align*}
    J(\pi_\phi) & \ge J(\pi_b) - \frac{ r^*}{(1-\gamma)^2} \sqrt{ C_\lambda + \sqrt{(V - \log \delta)/n} }
\end{align*}

The assumptions adopted for this result are listed below. Similar results can be derived under more general forms of these assumptions, but ours   greatly simplify the form of our theorem and its proof. 
\begin{itemize} \setlength\itemsep{0.2em}
    \item[(A1)] The complexity of the function class $\Pi$ of possible  
policy networks $\pi_\phi$ (in terms of the log-likelihood loss $\log \pi$) is bounded by $V$. Here $V$ is defined as the extension of the VC dimension to real-valued functions with unbounded loss, formally detailed in Section III.D of \citetsi{vapnik1999overview}. 
Similar results hold under alternative complexity measures $V$ from the literature on empirical process theory for density and f-divergence estimation  \citesi{nguyen2010estimating,geer2000empirical,van1996weak}.

\item[(A2)] Rewards in our environment are bounded such that $r(s,a) \le r^*$ for all $s \in \mathcal{S}, a \in \mathcal{A}$.
\item[(A3)] Our learned Q networks are bounded such that $|Q_\theta(s,a)| < B$ for all $s,a$. 
\item[(A4)] The likelihoods of our learned $\pi_\phi$ are bounded such that $|\log \pi_\phi(a |s)| < L$ for all $s,a$.
\item[(A5)] Each policy update step is carried out using the full dataset rather than mini-batch.
\end{itemize}

\begin{proof}
Define $\displaystyle r^* = \max_{a,s} |r(s,a)|$, and let  $d^{\pi_{b}}$ denote the marginal distribution of states encountered by acting according to $\pi_b$ starting from the initial state distribution $\mu_0$. Thus $d^{\pi_{b}}$ describes the probability distribution underlying the states present in our dataset $\mathcal{D}$.   
Recall the \emph{total variation} distance between probability distributions $p$ and $q$ is defined as: $\displaystyle \text{TV}(p,q) = \int |p(x)-q(x)| \mathrm{d}x $.

From equation (18) in the Proof of Proposition 2 (Appendix A.2) from  \citesi{Wang2018Proof}, we have:
\begin{align*}
J(\pi_\phi) & \ge J(\pi_b) - \frac{ r^*}{(1-\gamma)^2} \mathbb{E}_{s\sim d^{\pi_{b}}} \left[  \text{TV} \Big( \pi_\phi(\cdot \mid s) , \pi_b(\cdot \mid s) \Big) \right] \\
& \ge J(\pi_b) -  \frac{r^*}{\sqrt{2} (1-\gamma)^2} \mathbb{E}_{s\sim d^{\pi_{b}}}\Big[\sqrt{\text{KL}   \Big( \pi_b(\cdot\mid s), \pi_\phi(\cdot\mid s) \Big) }\Big]
\\
& \ge J(\pi_b) -  \frac{ r^*}{ \sqrt{2} (1-\gamma)^2} \sqrt{ \mathbb{E}_{s\sim d^{\pi_{b}}} \left[  \text{KL}  \Big( \pi_b(\cdot\mid s), \pi_\phi(\cdot\mid s) \Big) \right] }
\end{align*}
where we used Pinsker's inequality in the second line  (c.f.\ \citesi{csiszar2011}), and the last line is an application of Jensen's inequality for the concave function $f(x)=\sqrt{x}$.

By assumption (A1), each update of our policy network $\pi_\phi$ in CDC is produced via:
\begin{align*}
   \phi \leftarrow \argmax_{\phi} \hspace*{-2mm}
     \sum_{(s,a) \in \mathcal{D}, \hat{a} \sim \pi_\phi(\cdot|s)}
    \Big[
    \overline{Q}_{\theta_t}(s,  \hat{a})
       {  \ + \ \lambda \cdot \log \pi_\phi(a|s) }  \Big]
\end{align*}
where $\theta_t$ denotes the current parameters of our Q networks in iteration $t$ of CDC. 
Each of these penalized optimizations can be equivalently formulated using a  hard constraint, i.e., there exists constant $C_{\lambda,\theta_t} > 0$ (for which $\lambda$ is the corresponding Lagrange multiplier), such that the following optimization leads to the same $\phi$:
\begin{align*}
   \phi \leftarrow &  \argmax_{\phi} \hspace*{-2mm}
     \sum_{(s,a) \in \mathcal{D}, \hat{a} \sim \pi_\phi(\cdot|s)}
    \Big[
    \overline{Q}_{\theta_t}(s,  \hat{a})  \Big] \\ 
   &  \text{subject to }  \E_{{(s,a)} \sim \mathcal{D}} \left[ \log \pi_\phi(a|s) \right] \ge C_{\lambda,\theta_t}
\end{align*}

\textbf{Note:} Throughout, $\E_{{(s,a)} \sim \mathcal{D}}$ is an  \emph{empirical} expectation over dataset $\mathcal{D}$, whereas $\E_{\pi_b}$ denotes true expectations over the underlying distribution of the behavior policy.  
Since all $Q_{\theta_t}$ are bounded by (A3), so must be  
\begin{equation}
C_\lambda^* := \max_t \{ C_{\lambda,\theta_t} \}
\label{eq:clog}
\end{equation} 
Thus, in every iteration of CDC, the resulting $\pi_\phi$ must satisfy: 
\begin{equation}
    \E_{{s,a} \sim \mathcal{D}} \left[ \log \pi_\phi(a|s) \right] \ge C_{\lambda^*}\ . 
    \label{eq:C_lambda_star}
\end{equation}

Finally, we conclude the proof by using Lemma \ref{lem:probconverge} to replace the bound on the empirical likelihood values with a bound on the underlying KL divergence from the data-generating behavior policy  distribution.  
\end{proof}

\begin{lemma} \label{lem:probconverge}
Suppose $\E_{{(s,a)} \sim \mathcal{D}} [ \log \pi_\phi(a|s) ] \ge C_\lambda^*$. Then with probability $\ge 1- \delta$:
\begin{align*}
\mathbb{E}_{s\sim d^{\pi_{b}}} \Big[  \textnormal{KL}  \big( \pi_b(\cdot\mid s), \pi_\phi(\cdot\mid s) \big) \Big] \le C_\lambda + \sqrt{(V - \log \delta)/n}
\end{align*}
where $n=|\mathcal{D}|$, and $d^{\pi_{b}}$ denotes the marginal state-visitation distribution under the behavior policy, and: 
\begin{equation}
C_\lambda := C_b -  C_\lambda^* 
\label{eq:clambda}
\end{equation}
for $C_\lambda^*$ defined in \eqref{eq:clog} and constant $C_b := \EE_{s\sim d^{\pi_{b}},a \sim \pi_b(\cdot|s)}\big[ \log \pi_b(a\mid s) \big]$.
\end{lemma}

\begin{proof}
A classical result in statistical learning theory (Theorem (23) in Section III.D of \citetsi{vapnik1999overview}) states that the following bound simultaneously holds for all $\pi_\phi \in \Pi$ with probability $1- \delta$:
\begin{equation}
\label{eq:empiricalprocess}
\mathbb{E}_{s\sim d^{\pi_{b}},a \sim \pi_b(\cdot|s)} [ \log \pi_\phi(a|s) ] 
\ge \mathbb{E}_{ (s,a) \sim \mathcal{D} } [ \log \pi_\phi(a|s) ] - \sqrt{(V - \log \delta)/n}
\end{equation}

Recall $V$ measures the complexity of function class $\Pi$ (with respect to the log-likelihood loss) and here is defined as the extension of the VC dimension to real-valued functions with unbounded loss from Section III.D of \citetsi{vapnik1999overview}. 
We now write: 
\begin{align*}
    \EE_{s\sim d^{\pi_{b}}}\Big[\text{KL}  \big( \pi_b(\cdot\mid s), \pi_\phi(\cdot\mid s) \big)\Big]
 & = \EE_{s\sim d^{\pi_{b}},a \sim \pi_b(\cdot|s)}\big[ \log \pi_b(a\mid s) \big]\\
 &\quad - \EE_{s\sim d^{\pi_{b}},a \sim \pi_b(\cdot|s)}\big[ \log \pi_\phi(a\mid s) \big] 
 \\ 
&  = C_b - \EE_{s\sim d^{\pi_{b}},a \sim \pi_b(\cdot|s)}\big[ \log \pi_\phi(a\mid s) \big] \tag*{by definition of constant $C_b$}
 \\ 
 & \le C_b  - \left( \E_{{(s,a)} \sim \mathcal{D}} \big[ \log \pi_\phi(a \mid s) \big] - \sqrt{(V - \log \delta)/n} \right) 
 \\
\tag*{by the empirical process bound in \eqref{eq:empiricalprocess} } 
 \\
 & \le C_b - C_\lambda^* + \sqrt{(V - \log \delta)/n} 
 \\
  \tag*{from (\ref{eq:C_lambda_star})}
\\
 & = C_\lambda + \sqrt{(V - \log \delta)/n} 
 \\
  \tag*{from (\ref{eq:clambda})}
\end{align*}
allowing us to conclude the proof.

\end{proof}

\subsection{Proof of Theorem \ref{thm:oe}. }\label{app:proof_oe}
\textbf{Theorem \ref{thm:oe}:} Define $\textnormal{OE}_{\textit{ag}}$ as the resultant overestimation bias when performing the maximization step by an agent \textit{ag}: $\expectation[\max_{a}Q_{\theta}(s,a)] - \max_{a} Q^{*}(s,a)$. 
Here $Q_{\theta}$ denotes the estimate of true Q-value ($Q^*$)  learned by $\textnormal{ag}$, which may use \textbf{\algname{}} (with $\eta >0$) or a \textbf{baseline} that uses \algref{alg:ours} with $\eta=0$ (with the same value of $\lambda$), and the expectation is taken over the randomness of the underlying dataset, as well as the learning process. Under the assumptions stated below, there co-exist constants $L_1$ and $L_2$ such that
\begin{align*}
\textnormal{OE}_{\textnormal{CDC}}\leq L_{1} - \eta L_{2} \leq \textnormal{OE}_{\textnormal{\regular}} \ .
\end{align*}

This result relies on the following assumptions:

\begin{itemize}
    \item[(A1)] For a specific state-action pair $\langle s,a_{ID}\rangle$ in our dataset $\mathcal{D}$, we assume a function approximator parameterized by $\theta=\langle \textbf q, V, Q_{ID} \rangle$ as follows:
\[Q_{\theta}(s,a) = \begin{cases} 
      \sum_{i=1}^{m} \mathbbm{1}\big(a\in\mathcal{A}_{i}(s)\big)\big(V(s)+ q_i\big) & a\in A\setminus a_{ID} \\
      Q_{ID} & \textrm{otherwise}\\
   \end{cases}
\]
where sets $\mathcal{A}_i~\forall i \in [1,m]$ form a disjoint non-empty partition of the action space $A$. Here $q_i$ could be thought of as a subset-dependent value, while $V$ is a common (subset-independent) value that quantifies the overall quality of a particular state $s$, also considered by~\citet{wang2016dueling}.
Note that such discretization schemes are commonly assumed in the RL literature \citesi{sutton2018reinforcement}, and that such function approximators can closely approximate all continuous functions up to any desired accuracy as $m$ increases. At the cost of a more complex analysis, we may have alternatively assumed Lipschitz continuous function approximators, since they are closely related to discretization through the notion of covering number.
    \item[(A2)] Following \citet{ThrunSchwartz1993}\footnote{Another paper with this assumption is the work of \citet{lan2019maxmin} who did not state this assumption, but to get their second equality on their page 12, the assumption A2 needs to hold.}, we study overestimation in a particular state $s$ with $\forall a~Q^{*}(s,a)=C$. Previous work studied this case because it is the setting where maximal overestimation occurs (due to maximal Q-value overlap). Stated differently, not assuming A2 means our theorem still provides the following bound: $\max(\textnormal{OE}_{\textnormal{CDC}})\leq\max(\textnormal{OE}_{\textnormal{\regular}})$ where the maximum is taken over all possible true Q functions.
    \item[(A3)] We assume $\forall s, i:$ the $q_i$ are independent and are distributed uniformly in $[-L_{1},L_{1}]$. This assumption is also adopted from the literature~\citesi{ThrunSchwartz1993, lan2019maxmin} and greatly simplifies our analysis. We further assume $ Q_{\theta}(s,a_{ID})=Q_{ID}$, where $Q_{ID}$ is distributed uniformly at random in $[-\alpha L_{1},\alpha L_{1}]$ for $\alpha\ll 1$, reflecting the conviction that our $Q$ estimates should generally be more accurate for the previously observed state-action pairs.
    
    \item[(A4)] We assume that Step \ref{step:value} of \algref{alg:ours} proceeds by first updating $\theta$ based on taking a gradient-step towards minimizing $\eta \Delta(s,a)$ term (see \eqref{eq:delta}): $\theta\leftarrow\theta - \mu \eta \nabla_{\theta} \Delta(s,a)$, followed by performing the maximization step in minimizing the TD error. Separating the update into two steps simplifies our analysis while remaining to be a reasonable way to perform Step \ref{step:value} of \algref{alg:ours}. Here, $\mu$ is the step-size. 
    \item[(A5)] We assume the policy $\pi_{\phi}$ assigns non-zero probability to at least one action from each $\mathcal{A}_i$. 
\end{itemize}
\begin{proof}
We begin by quantifying the overestimation bias for the baseline case. By denoting $a_{i}$ to be a member of $\mathcal{A}_i(s)$, and in light of our function approximator (assumption A1), we can write:
\begin{align*}
    \expectation\Big[\max_{a\in\mathcal{A}} Q_{\theta}(s,a)\Big]&=\expectation\Big[\max \big\{Q_{\theta}(s,a_{ID}),  Q_{\theta}(s,a_{1}),\cdots,  Q_{\theta}(s,a_{m})\big\}\Big]\\
    & \quad \textrm{(assumptions A2, and A3)}\\
    &=\expectation\Big[\max \{Q_{ID},C+  q_1,\cdots,C+  q_m\}\Big]\\
    & \quad \textrm{(assumption A3)}\\
    & \geq   \expectation\Big[\max \big\{C-\alpha L_{1},  C+q_1,\cdots,  C+q_m\big\}\Big]\\
    &=C+\expectation\Big[\max \big\{-\alpha L_{1},  q_1-\alpha L_{1}+\alpha L_{1},\cdots,  q_m-\alpha L_{1}+\alpha L_{1}\big\}\Big]\\
    &=C-\alpha L_{1}+\expectation\Big[\max \big\{0,  q_1+\alpha L_{1},\cdots,  q_m+\alpha L_{1}\big\}\Big] \\
    &= C-\alpha L_{1}+m\int_{x:-\alpha L_1}^{L_1}\frac{(x+\alpha L_1)}{2L_1} \Big(\int_{y:-L_1}^x \frac{1}{2L_1}dy\Big)^{m-1}dx\\
    &= C-\alpha L_{1}+m\int_{x:-\alpha L_1}^{L_1} \frac{(x+\alpha L_1)}{2L_1} \big(\frac{x+L_{1}}{2 L_1}\big)^{m-1}dx \ .
\end{align*}
In the penultimate step, the equality holds because we only need to consider cases where at least one of the random variables is bigger than $-\alpha L_1$, because otherwise the maximum is 0, thus not affecting the expectation. We broke down the expectation to $m$ cases, where in each case the maximizing noise is at least $-\alpha L_1$ (the integral over $x$), and the remaining $n-1$ variables are smaller than the maximizing one (the integral over $y$).

Using a change-of-variable technique~($z=\frac{x+L_1}{2L_1}$), we can then write:
\begin{align}
    \expectation\Big[\max_{a\in\mathcal{A}} Q_{\theta}(s,a)\Big]&\geq C-\alpha L_{1}+m\int_{x:-\alpha L_1}^{L_1} \frac{(x+\alpha L_1)}{2L_1} \big(\frac{x+L_{1}}{2 L_1}\big)^{m-1}dx\nonumber\\
    &=C-\alpha L_{1}+m L_1 \int_{y:\frac{1-\alpha}{2}}^{1}\big(2y-(1-\alpha)\big)y^{m-1}dy\nonumber\\
    &= C-\alpha L_{1}+\underbrace{mL_1 \big(\frac{2}{m+1}-\frac{1-\alpha}{m}\big)+L_1\big(\frac{(\frac{1-\alpha}{2})^{m+1}}{m+1})}_{:=f(L_1,m,\alpha)}\label{eq:f_definition} \ ,
\end{align}
allowing us to write that: $\textnormal{OE}_{\textnormal{\regular}}\geq -\alpha L_{1}+f(L_1,m,\alpha).$ Notice that $\lim_{m\rightarrow\infty} f(L_1,m,\alpha)=L_1 + \alpha L_1$.
For example, with $\alpha=0$ (for which the bound is tight), the overestimation bias of the baseline for $m=2$ is $C+f(2,L_1,0)-\max_{a}Q(s,a)=C+f(2,L_1,0)-C=f(2,L_1,0)=\frac{5L_1}{12}$ and monotonically increases and converges to $L_1$ as $m\rightarrow\infty$.

We now move to the case where we perform the \algname{} update prior to the maximization step (assumption A4). The update proceeds by first choosing the maximum OOD action, which given assumption A5 corresponds to: $\max_{i\in[1,m]}   Q_{\theta}(s,a_i)$.

Now define ${\boldsymbol{\varepsilon_{+}} := \max\big\{0,\max_{i}Q_{\theta}(s,a_i)-Q_{\theta}(s,a_{ID})\big\}}$, we have:

\begin{align*}
    \expectation\Big[\boldsymbol{\varepsilon_{+}}\Big]&=\expectation\Big[\max\big\{0,\max_{i}Q_{\theta}(s,a_i)-Q_{\theta}(s,a_{ID})\big\}\Big]\\
    &=\expectation\Big[\max\big\{0,\max_{i}Q_{\theta}(s,a_i)-(Q_{ID})\big\}\Big]\\
    &\quad \textrm{(from Assumption A3, namely: $Q_{ID}\leq C+\alpha L_1$)}\\
    &\geq\expectation\Big[\max\big\{0,\max_{i}Q_{\theta}(s,a_i)-C-\alpha L_1\big\}\Big]\\
    &=\expectation\Big[\max\big\{0,\max_{i}C+q_i-C-\alpha L_1\big\}\Big]=\expectation\Big[\max\big\{0,\max_{i}q_i-\alpha L_1\big\}\Big]\\
    &=\expectation\Big[\max\big\{2\alpha L_1,\max_{i}q_i+\alpha L_1\big\}\Big]-2\alpha L_1\\
    &\geq \expectation\Big[\max\big\{0,\max_{i}q_i+\alpha L_1\big\}\Big]-2\alpha L_1\\
    &= \expectation\Big[\max \big\{0,  q_1+\alpha L_{1},\cdots,  q_m+\alpha L_{1}\big\}\Big]-2\alpha L_1\\
    & \quad \textrm{(from our analysis of the baseline case and (\ref{eq:f_definition}))}\\
    &=f(L_1,m,\alpha)-2\alpha L_1 \ .
\end{align*}

Without loss of generality, assume that the maximizing action  is $a_1 = \argmax_a Q_\theta(s, a)$. Thus we can quantify the expected \algname{} update as follows (see also \eqref{eq:delta_grad}):
\begin{align}
    \Big( \boldsymbol{\nabla}_{\theta}   Q_{\theta}(s, a_1) - \boldsymbol{\nabla}_\theta   Q_{\theta}(s, a_{ID})\Big) ~  \expectation\Big[\boldsymbol{\varepsilon_{+}}\Big]
    &\geq\Big( \boldsymbol{\nabla}   Q_{\theta}(s, a_1) - \boldsymbol{\nabla}   Q_{\theta}(s, a_{ID})\Big)\Big(f(L_1,m,\alpha)-2\alpha L_{1}\Big)
\label{eq:avg_epsilon}
\end{align}

Given the form of our function approximator, by performing the updates (assumption A4) $$\begin{bmatrix} \textbf q \\ V(s) \\ Q_{ID} \end{bmatrix} \leftarrow \begin{bmatrix} \textbf q \\ V(s) \\ Q_{ID} \end{bmatrix} - \mu\eta \Big[\Big( \boldsymbol{\nabla}_\theta   Q_\theta(s, a_1) - \boldsymbol{\nabla}_\theta Q_\theta(s, a_{ID})\Big) ~  \boldsymbol{\varepsilon_{+}}\Big]\ ,$$ 
 CDC will:
\begin{itemize}
    \item inflate the value of $Q_\theta(s,a_{ID})$ by at least $\mu\eta \Big(f(L_1,m,\alpha)-2\alpha L_1\Big)$. This is due to updating $Q_{ID}$.
    \item deflate $Q_\theta(s,a_1)$ by at least $2\mu\eta \Big(f(L_1,m,\alpha)-2\alpha L_1\Big)$. This is due to updating $q_1$ and $V(s)$.
    \item deflate $Q_\theta(s,a_i)~\forall i\neq 1$ by at least $\mu\eta \Big(f(L_1,m,\alpha)-2\alpha L_1\Big)$, due to updating $V(s)$.
\end{itemize}

Now notice that, based on (\ref{eq:avg_epsilon}), we will at least subtract $\eta\mu \Big(f(L_1,m,\alpha)-2\alpha L_1\Big)$ from all $Q_{\theta}(s,a)$, whose value could at most have been $C+L_{1}$ prior to the update (assumption A3). Therefore we can claim:
\begin{align*}
\expectation\Big[\max_{a}   Q_\theta(s,a)\Big]\leq C + L_1 - \eta\underbrace{\mu \Big(f(L_1,m,\alpha)-2\alpha L_1\Big)}_{L_2} \ ,
\end{align*}
and that:
\begin{align*}
\expectation\Big[\textnormal{OE}_{\textrm{CDC}}\Big]=\expectation\Big[\max_{a}   Q_\theta(s,a)\Big]-\max_{a} Q^{*}(s,a)\leq C+L_1 - \eta L_2-C=L_1 - \eta L_2
\end{align*}

Further, for $\textrm{OE}_{\textrm{CDC}}\leq \textrm{OE}_{\textrm{baseline}}$ to hold, we need that:
\begin{align*}
f(L_1,m,\alpha)\geq\frac{(1+2\alpha\eta\mu +\alpha)L_1}{1+\mu\eta}\ ,
\end{align*}
For sufficiently small $\mu$, i.e. $\mu<\frac{1-\alpha}{2\alpha\eta}$, and sufficiently large $m$, i.e, $m\geq \frac{1}{1/2-\frac{1}{1+\eta\mu}}$, we get the desired result: $\textnormal{OE}_{\textrm{CDC}}\leq\textnormal{OE}_{\textrm{baseline}}$, allowing us to conclude the proof.
\end{proof}
\clearpage
\bibliographysi{references}
\bibliographystylesi{abbrvnat}

\end{document}